\ifdefined\pdfoutput\pdfoutput=1\fi % force PDF route on pdfTeX/LuaTeX; XeTeX has no \pdfoutput and outputs PDF anyway
\documentclass{article}

\usepackage[
    top=1in,
    bottom=1in,
    left=1.25in,
    right=1.25in
]{geometry}
\usepackage{amsmath}
\usepackage{hyperref}
\usepackage{cleveref}
\usepackage{algorithm}
\usepackage{amssymb}
\usepackage{graphicx}
\usepackage{booktabs}
\usepackage{makecell}
\usepackage{algpseudocode}
\usepackage{tikz}
\usepackage[dvipsnames]{xcolor}
\graphicspath{{figures/}}

\providecommand{\citep}{\cite}

\usepackage{amsthm}

\newtheorem{theorem}{Theorem}[section]

\newtheorem{lemma}[theorem]{Lemma}
\newtheorem{assumption}[theorem]{Assumption}
\newtheorem{proposition}[theorem]{Proposition}

\theoremstyle{definition}
\newtheorem{definition}[theorem]{Definition}

\theoremstyle{remark}

\crefname{assumption}{Assumption}{Assumptions}

\usepackage{mathtools}

\DeclareMathOperator{\Cov}{Cov}
\DeclareMathOperator{\Var}{Var}
\DeclareMathOperator{\Tr}{Tr}

\newcommand{\trans}{\mathsf{T}}

\newcommand{\mR}{\mathbb{R}}
\newcommand{\mN}{\mathcal{N}}

\newcommand{\mbE}{\mathbb{E}}
\newcommand{\mbP}{\mathbb{P}}
\newcommand{\mbQ}{\mathbb{Q}}

\newcommand{\md}{\mathop{}\!\mathrm{d}}

\newcommand{\mH}{\mathcal{H}}

\newcommand{\mM}{\mathcal{M}}

\newcommand{\norm}[1]{\left\| #1 \right\|}
\newcommand{\abs}[1]{\left| #1 \right|}
\newcommand{\KL}{D_{\mathrm{KL}}}
\newcommand{\KLdiv}[2]{\KL\left(#1\,\middle\|\,#2\right)}% KL divergence with spaced middle bar
\newcommand{\TV}[2]{\mathrm{TV}\left(#1,#2\right)}% KL divergence with spaced middle bar

\newcommand{\me}{\mathrm{e}}

\newcommand{\Nreg}{{N_\mathrm{reg}}}

\newcommand{\mf}{\mathrm{f}}
\newcommand{\ma}{\mathrm{a}}
\newcommand{\pf}{\pi^{\mathrm{f}}}
\newcommand{\pa}{\pi^{\mathrm{a}}}

\newcommand{\hpa}{\hat{\pi}^{\mathrm{a}}}
\newcommand{\tildepa}{\tilde{\pi}^{\mathrm{a}}}

\crefname{theorem}{Theorem}{Theorems}            \Crefname{theorem}{Theorem}{Theorems}
\crefname{lemma}{Lemma}{Lemmas}                  \Crefname{lemma}{Lemma}{Lemmas}
\crefname{proposition}{Proposition}{Propositions}\Crefname{proposition}{Proposition}{Propositions}
\crefname{corollary}{Corollary}{Corollaries}     \Crefname{corollary}{Corollary}{Corollaries}
\crefname{definition}{Definition}{Definitions}   \Crefname{definition}{Definition}{Definitions}
\crefname{assumption}{Assumption}{Assumptions}   \Crefname{assumption}{Assumption}{Assumptions}
\crefname{remark}{Remark}{Remarks}               \Crefname{remark}{Remark}{Remarks}
\AddToHook{env/theorem/begin}{\crefalias{theorem}{theorem}}
\AddToHook{env/lemma/begin}{\crefalias{theorem}{lemma}}
\AddToHook{env/proposition/begin}{\crefalias{theorem}{proposition}}
\AddToHook{env/corollary/begin}{\crefalias{theorem}{corollary}}
\AddToHook{env/definition/begin}{\crefalias{theorem}{definition}}
\AddToHook{env/assumption/begin}{\crefalias{theorem}{assumption}}
\AddToHook{env/remark/begin}{\crefalias{theorem}{remark}}

\usepackage[defaultcolor=OrangeRed,commentmarkup=margin]{changes}

\usepackage{authblk}

\newcommand{\email}[1]{\texttt{#1}}

\title{Ensemble Controlled-Flow Filtering for Implicit \\ Data Assimilation

}

\author[1]{Zhuoyuan Li%
\thanks{Equal contributions.}%
\thanks{Corresponding to Z.L. (\email{zy.li@nus.edu.sg}).
}%
}
\author[1]{Yue Zhao\protect\footnotemark[1]}
\author[2]{Ming Li}

\affil[1]{Institute for Functional Intelligent Materials, National University of Singapore, Singapore 117544, Singapore}
\affil[2]{School of Mathematical Sciences, Fudan University, Shanghai 200433, China}

\date{}

\begin{document}
\maketitle

\begin{abstract}
Data assimilation estimates the state of a dynamical system from model forecasts and incoming observations. Many observation mechanisms, however, are many-to-one, implicit, non-smooth, or accessible only through simulation, and need not provide the residual structures or likelihood guidance required by existing ensemble filters. We introduce implicit data assimilation, in which the analysis law is defined as an energy tilt of the forecast distribution. We then propose the Ensemble Controlled-flow Filter (EnCF), which realizes this update through a stochastic controlled flow and learns the observation-dependent control by adjoint matching from terminal energy gradients. For simulator-defined observations, EnCF-LF learns a surrogate conditional energy from samples and applies the same controlled-flow solver. We prove ideal exactness, derive a one-step error decomposition, and establish non-accumulation of local errors under filter stability. Numerical results show that Kalman-type filters remain preferable for smooth additive-Gaussian observations, while the proposed methods are better suited to non-Gaussian, many-to-one, multimodal, and implicit observation models.

\end{abstract}

\providecommand{\keywords}[1]{
    \vspace{0.2cm} % Adds a little space below the abstract
    \noindent \small \textbf{\textit{Keywords:}} #1
}

\begin{keywords}
{data assimilation, ensemble filtering, data-driven modeling, stochastic optimal control, adjoint matching}
\end{keywords}

% \begin{MSCcodes}
% 65C30, %: Stochastic differential and integral equations
% 65C05, %: Monte Carlo methods
% 62M20, %: Prediction; filtering
% 93E11, %: Filtering
% 93E20 %: Optimal stochastic control
% % 65C05, 62M20, 93E11, 65C30
% \end{MSCcodes}

\providecommand{\proofbox}{\qed}

% Shared main text (sections + appendix). Compiled by main_els.tex and main_siam.tex.
\section{Introduction}
\label{sec:intro}
Multiscale modeling in scientific computing concerns the relation between an
underlying dynamical description and the quantities accessible to computation,
experiment, or measurement
\citep{weinan2011principles,pavliotis2008multiscale}.
From a forward perspective, one derives, approximates, or learns reduced
mesoscopic or macroscopic dynamics for relevant observables
\citep{chen2024constructing,li2026hypothesis}.
A complementary inverse perspective seeks to recover the evolving underlying
state from partial, noisy, or indirect observations rather than to construct a
closed model for those observables. This is the setting of data assimilation
(DA), which combines dynamical models with observational data to infer states,
trajectories, and other latent quantities
\citep{asch2016da,law2015da}. Originating in numerical weather prediction
\citep{kalnay2003}, DA is now widely used in oceanography, geophysics, ecology,
and other partially observed dynamical systems
\citep{martin2024da-ocean,carrassi2018review,luo2011da-ecological,Wander2024da-ecological}.

In this work, we focus on the sequential data-assimilation setting. Sequential DA
tracks the evolving state by alternating between a forecast step, which propagates
the current state uncertainty through the dynamical model, and an analysis step,
which incorporates newly available observations \citep{asch2016da,law2015da}.
Among numerical realizations of this recursion,
ensemble-based methods are the scalable approach most widely used in large-scale systems \citep{evensen2009}.
They represent the forecast uncertainty by a finite
collection of sample states and update this ensemble whenever new data arrive.
Once a
forecast ensemble has been generated, the central numerical task considered here
is the analysis transformation mapping from the forecast distribution $\pf$ to the Bayesian analysis distribution $\pa\propto\ell(\cdot;y)\pf$ by taking into account the likelihood model $\ell$ and the observation data $y$.

Existing ensemble filters differ mainly in the computational interface through which they use the observation model. For additive-Gaussian observations, the EnKF constructs its analysis update from empirical first and second moments \citep{evensen2009,hunt2007letkf}. The bootstrap particle filter evaluates the likelihood at forecast particles and then resamples according to the resulting importance weights \citep{Arulampalam2002PF-tutorial,vanLeeuwen2009PF}. Among recent flow methods, EnSF evaluates the observation-likelihood score along a diffusion path, whereas EnFF uses likelihood-reweighted Monte-Carlo guidance or a localized likelihood-gradient approximation \citep{bao2024ensf,enff}. These interfaces are effective when the observation model provides the required quantities, but many scientific sensing mechanisms do not. The map from state to observation can be non-smooth, many-to-one, implicit, or available only through a simulator. For example, a quantized sensor reports only the interval in which a value falls \citep{anderson2010rhf}, an intensity measurement records magnitude but discards phase \citep{shechtman2015phase}, and a radiative-transfer code may provide forward samples without a tractable likelihood \citep{cranmer2020sbi}. These observations can still define informative analysis updates, but they may not supply a reliable additive residual, evaluable likelihood weights, or analytic likelihood guidance along a transport path. This motivates separating the distribution targeted by the analysis from the numerical mechanism used to realize the update.

We encode the observation information by a scalar energy \(J(x;y)\) and define the target analysis law by
\(\pa(x\mid y)\propto\pf(x)\exp[-J(x;y)]\). When \(J(x;y)=-\log p(y\mid x)\), this is the standard Bayesian posterior. More generally, \(J\) may be a compatibility energy that is not a normalized observation density, or a surrogate learned from simulator samples. We call this formulation \emph{implicit data assimilation}. Because the target is defined as a tilt of the forecast law rather than as a Gaussian moment update, it can represent multiple state configurations compatible with a many-to-one observation.

\begin{figure}
    \centering
    \includegraphics[width=\linewidth]{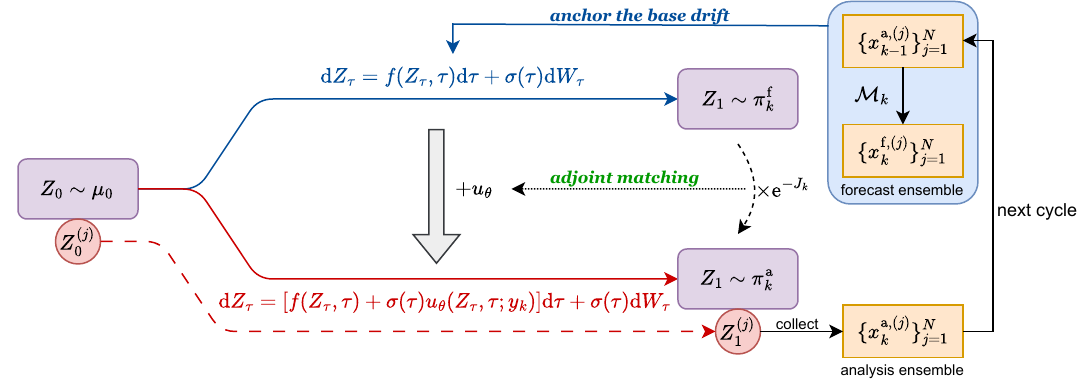}
    \caption{Illustration of the ensemble controlled-flow framework.}
    \label{fig:encf-illustration}
\end{figure}

To approximate this target with a finite ensemble, the ensemble controlled-flow filter introduces an auxiliary transport time and steers a base diffusion whose terminal law is the forecast distribution. Under the memoryless construction and exact optimal control, the controlled terminal law equals the energy-based analysis law. In the practical method, the optimal control is approximated by adjoint matching \citep{domingoenrich2025adjoint}: terminal gradients \(\nabla J(Z_1;y)\) are propagated backward through the base dynamics to form regression targets for the control. EnCF therefore requires a differentiable analytic energy, but it does not prescribe an observation-likelihood gradient at every transport-time step.
\Cref{fig:encf-illustration} summarizes this design. The observation reaches the analysis only through
the energy interface, and one assimilation cycle realizes the tilt as a controlled transport of the
forecast ensemble.

When the observation is represented by a differentiable analytic energy, EnCF uses that energy
directly. When only simulator samples are available, when a Gaussian residual model gives the wrong
conditional law, or when direct gradients are unavailable or uninformative, its likelihood-free variant EnCF-LF learns a smooth surrogate energy function and uses its gradient in the same adjoint-matching solver. We analyze EnCF as a numerical
scheme for the Bayes filtering map. It is exact with an exact base flow and optimal control, its one-step
error splits into separately refinable approximation terms, and these local errors do not accumulate when
the exact filter is stable. 
The experiments indicate that the controlled-flow formulation is most useful in
regimes with non-Gaussian, multimodal, many-to-one, or implicit observation models.
The likelihood-free variant addresses the weaker-access case in which only
forward simulator samples are available, by learning the energy used by the same
controlled-flow solver.

Our contributions are summarized as follows.
\begin{enumerate}
\item \textbf{Implicit data assimilation.}
We formulate the analysis update through a scalar observation energy \(J\),
defining the target law as the exponential tilt
\(\pa\propto\pf\me^{-J}\) of the forecast distribution. This formulation
separates the specification of the observation model from the numerical
analysis mechanism and accommodates explicit likelihoods, implicit
constraints, joint models, and learned surrogate energies without requiring
an explicit forward observation map.
\item \textbf{Ensemble controlled-flow filter.}
We introduce the Ensemble Controlled-flow Filter, which realizes the
energy-defined analysis as a stochastic optimal-control problem. The
observation-dependent feedback control is approximated by adjoint matching
from terminal energy gradients. Its likelihood-free variant, EnCF-LF, learns
a conditional surrogate energy from simulator samples and applies the same
controlled-flow solver.

\item \textbf{Error analysis.}
We establish exactness in the ideal limit and derive a one-step
total-variation error decomposition that separates the effects of the
empirical base flow, control approximation, surrogate energy, and time
discretization. Under exponential stability of the exact filtering recursion,
we further show that uniformly bounded local errors yield a uniformly bounded
multi-step filtering error.
\item \textbf{Regime-based evaluation.}
We compare EnCF and EnCF-LF with classical and flow-based ensemble filters
across additive-Gaussian, non-Gaussian, many-to-one, multimodal, and implicit
observation models. The results empirically validate the intended regime of the method:
tuned Kalman filters remain preferable for smooth additive-Gaussian
observations, whereas controlled transport becomes advantageous when the
analysis geometry is non-Gaussian, multimodal, or implicitly defined.
\end{enumerate}

\section{Background and related work}
\label{sec:related}

Consider the discrete-time state-space model
\begin{equation}\label{eq:ssm}
    x_k=\mM_k(x_{k-1};\eta_k)\in\mR^n,\qquad\eta_k\sim\mathcal D_k,
\end{equation}
where \(\mM_k\) is a (possibly stochastic) prediction operator and \(\eta_k\) denotes the model noise distributed according to \(\mathcal D_k\).
At each time step \(k\), we observe the system \(x_k\) through an observation \(y_k\in\mR^m\) satisfying \(y_k\sim p(\cdot\mid x_k)\). We assume \(p(y_k\mid x_k,y_{<k})=p(y_k\mid x_k)\), i.e., the current observation is conditionally independent of previous observations given the current state.
We denote by
\begin{equation}
    \pf_k=\pf_k(x_k)\coloneqq p(x_k\mid y_{<k})\quad\text{and}\quad\pa_k=\pa_k(x_k)\coloneqq p(x_k\mid y_{\le k})
\end{equation}
the forecast and analysis distributions of the system state \(x_k\), respectively, where \(y_{<k}\) and \(y_{\le k}\) collect the observations before and up to step
\(k\). 

Sequential data assimilation seeks to recursively track the analysis distribution \(\pa_k\) by alternating between a prediction step and an analysis step. Specifically, for each step \(k\), the Chapman--Kolmogorov
equation gives
\begin{equation}
    p(x_k\mid y_{<k})=\int q_k^\mM(x_k\mid x_{k-1})p(x_{k-1}\mid y_{<k})\md x_{k-1},
\end{equation}
where \(q_k^\mM\) denotes the transition kernel induced by \(\mM_k\). 
Applying Bayes' rule then yields
\begin{equation}
    p(x_k\mid y_{\le k})=\frac{p(y_k\mid x_k,y_{<k})p(x_k\mid y_{<k})}{p(y_k\mid y_{<k})}=\frac{p(y_k\mid x_k)p(x_k\mid y_{<k})}{p(y_k\mid y_{<k})}.
\end{equation}
The second equality follows from the conditional-independence assumption above.
Consequently, sequential data assimilation can be summarized by the two-step recursion
\begin{equation}
\pf_k=(\mM_k)_\sharp\pa_{k-1}\quad(\text{prediction}),\qquad
\pa_k\propto \ell(\cdot;y_k)\pf_k\quad(\text{analysis}),
\label{eq:recursion}
\end{equation}
where \(\ell(x_k;y_k)\) denotes the likelihood density associated with \(p(y_k\mid x_k)\).

Classical ensemble-based filters realize this recursion by evolving an ensemble of \(N\) particles \(\{x_k^{(j)}\}_{j=1}^N\) over time. In the prediction step, each ensemble member is propagated independently according to
\begin{equation}
    x_k^{\mf,(j)}=\mM_k(x_{k-1}^{\ma,(j)};\eta_k^{(j)}),\quad \eta_k^{(j)}\sim\mathcal D_k.
\end{equation}
If the previous analysis ensemble approximates \(\pa_{k-1}\), independently propagating its members produces an ensemble approximation of \(\pf_k\).
Different ensemble filters mainly differ in how they transform the forecast ensemble members so as to approximate the analysis distribution \(\pa_k\propto \ell(\cdot;y_k)\pf_k\). % Below we briefly review two representative classical filters, which are also treated as the classical baselines in the experimental section.

For an additive-Gaussian likelihood
\begin{equation}
    \ell(x;y_k)=\mN(y_k;\mH(x),R_k),
\end{equation}
in which \(R_k\) is the observation-error covariance, the EnKF approximates the forecast distribution by a Gaussian whose moments match those of the forecast ensemble, and then shifts each member by the empirical Kalman gain, formulated as
\begin{equation}
x_k^{\ma,(j)}=x_k^{\mf,(j)}+P_{xy}^\mf\bigl(P_{yy}^\mf+R_k\bigr)^{-1}\bigl[y_k^{(j)}-\mH(x_k^{\mf,(j)})\bigr],\qquad y_k^{(j)}\sim\mN(y_k,R_k).
\label{eq:enkf}
\end{equation}
The matrices \(P^\mf_{xy},P^\mf_{yy}\) denote the cross-covariance and predicted observation covariance estimated from the ensembles, respectively. The EnKF recovers the exact analysis distribution as \(N\to\infty\) when the forecast is Gaussian and \(\mH\) is linear; finite \(N\) introduces covariance-estimation and perturbed-observation sampling errors. Outside the linear--Gaussian setting, an update determined by empirical first and second moments need not reproduce a non-Gaussian or multimodal analysis distribution.

The bootstrap particle filter instead preserves the empirical forecast measure and reweights it by the likelihood as
\begin{equation}
w_k^{(j)}=\frac{\ell(x_k^{\mf,(j)};y_k)}{\sum_{j'=1}^N \ell(x_k^{\mf,(j')};y_k)},\qquad
\tildepa_k\coloneqq\sum_{j=1}^N w_k^{(j)}\delta_{x_k^{\mf,(j)}}.
\label{eq:pf}
\end{equation}
The analysis ensemble \(\{x_k^{\ma,(j)}\}_{j=1}^N\) is then obtained by resampling from \(\tildepa_k\). For a fixed assimilation horizon, the bootstrap particle approximation converges to the filtering distribution as \(N\to\infty\) under the usual sequential-Monte-Carlo integrability assumptions, but it faces the curse of dimensionality and suffers from severe degeneracy in high-dimensional settings \citep{vanLeeuwen2009PF,snyder2008obstacles}.

{Flow-based filters construct stochastic or deterministic transports whose terminal laws approximate the analysis distribution. EnSF constructs a score-based diffusion transport \citep{bao2024ensf}, whereas EnFF constructs a flow-matching transport \citep{enff}. Related work extends score- and flow-based assimilation to partially observed systems, learned latent dynamics, and stochastic-interpolant models \citep{xiong2026ensflr,xiao2024ldensf,chen2025flowdas}. EnSF evaluates \(\nabla_z\log p(y_k\mid z)\) along the diffusion path, whereas EnFF uses either likelihood-reweighted Monte-Carlo guidance or a localized likelihood-gradient approximation.}
Data-driven methods address unknown or unresolved dynamics by learning or replacing forecast-model components from training data \citep{Brajard2021,ghosh2024danse,Penny2022,Li2025SOAD}. Other methods perform assimilation in learned reduced-order spaces \citep{Peyron2021,Amendola2021,cheng2023generalised,Li2024LAINR}.

The comparison above exposes three limitations of existing analysis interfaces. The EnKF reduces the update to empirical first and second moments and can therefore miss non-Gaussian or multimodal analysis distributions. The bootstrap particle filter is asymptotically consistent, but its observation-likelihood weights can degenerate in high dimensions. EnSF and EnFF avoid explicit likelihood-weighted resampling, yet they require an evaluable likelihood or likelihood gradient to prescribe observation-dependent guidance along the transport path. The data-driven and reduced-order methods above primarily address unknown or expensive dynamics and representation cost; here we instead assume access to forecast samples and isolate the analysis update. We seek an update that does not impose a Gaussian approximation, does not weight and resample forecast particles by the observation likelihood, and separates the specification of the observation from the transport mechanism. The next section defines the target energy-based analysis density \(\pa_k(x)\propto\pf_k(x)\exp[-J_k(x;y_k)]\) and constructs a controlled-flow approximation that learns the control by adjoint matching from terminal energy gradients rather than prescribing a likelihood-based guidance rule at every transport-time step.

% The discussion above isolates the obstruction that organizes the method. The analysis should be
% specified without committing to residuals, likelihood scores along a path, or Gaussian moments, and then
% realized by a transport that only asks for a tractable energy interface. The next section separates these
% two choices, first the energy-tilted analysis law, then the controlled diffusion and adjoint-matching
% approximation used to sample it.

\section{The ensemble controlled-flow filter}
\label{sec:method}

In this section, we formulate the analysis step as a controlled transport from the
forecast law to an energy-tilted analysis law. We begin by introducing implicit data
assimilation, in which the observation enters through a scalar energy rather than an
additive-Gaussian residual or a prescribed analytic representation of the likelihood.
We then cast the resulting analysis update as a stochastic optimal-control problem
and use this formulation to place several existing flow-based filters within a common
controlled-flow perspective~\citep{gao2026terminally}. Building on this viewpoint, we develop the ensemble
controlled-flow filter, which approximates the optimal feedback control through
adjoint matching. Finally, when an analytic observation energy is unavailable but
forward simulation of the observation mechanism is possible, we introduce the
likelihood-free variant EnCF-LF, which learns a surrogate energy from simulator samples
and uses the same controlled-flow solver.

\subsection{Implicit data assimilation via energy tilts}
\label{sec:ida}

Recall that the classical analysis step in \cref{eq:recursion} reweights the forecast
law by the likelihood \(\ell(\cdot;y_k)\). Particular numerical filters often require
more restrictive structures or computational interfaces, such as an additive-Gaussian
residual, a moment-based closure, or the likelihood score along a transport path.
We instead use a scalar energy as the computational interface between the observation
model and the analysis solver.

\begin{definition}[Implicit data assimilation]
\label{def:ida}
Let \(\pf_k\) be the forecast density and let
\(J_k(\cdot;y_k):\mR^n\rightarrow\mR\) be a measurable observation energy satisfying
\begin{equation}
0<
\mathcal Z_k(y_k)
\coloneqq
\mbE_{\pf_k}\left[\me^{-J_k(\cdot;y_k)}\right]
=
\int_{\mR^n}
\exp\left(-J_k(x;y_k)\right)\pf_k(x)\,\md x
<\infty.
\label{eq:energy-normalizer}
\end{equation}
The corresponding implicit data-assimilation analysis is the probability density
\begin{equation}
\pa_k(x)
=
\mathcal Z_k(y_k)^{-1}
\pf_k(x)\exp\left(-J_k(x;y_k)\right)
\propto
\pf_k(x)\exp\left(-J_k(x;y_k)\right).
\label{eq:ida}
\end{equation}
\end{definition}

The energy \(J_k(x;y_k)\) measures the compatibility between a candidate state \(x\)
and the observation \(y_k\), and is defined up to the addition of a term independent
of \(x\). It accommodates several observation interfaces.

When an evaluable likelihood is available, taking
\begin{equation}
J_k(x;y_k)=-\log \ell(x;y_k)
\label{eq:likelihood-energy}
\end{equation}
recovers the standard Bayesian analysis. This includes non-Gaussian and many-to-one
observation models and is not restricted to the additive-Gaussian setting.

The observation relation may instead be specified implicitly through
\(\mathcal R_k(x,y_k)\approx0\), without requiring an explicit forward representation
\(y_k=\mH(x_k)+\varepsilon_k\). A soft enforcement of this relation is obtained from
\begin{equation}
J_k(x;y_k)
=
\frac{\norm{\mathcal R_k(x,y_k)}^2}
{2\sigma_{\mathrm{tol}}^2},
\label{eq:constraint-energy}
\end{equation}
where \(\sigma_{\mathrm{tol}}>0\) determines the tolerance assigned to violations of
the relation.

More generally, any strictly positive compatibility potential
\(\Psi_k(x;y_k)\) can be represented by
\begin{equation}
J_k(x;y_k)=-\log\Psi_k(x;y_k),
\end{equation}
provided that \cref{eq:energy-normalizer} holds. In this form, \cref{eq:ida} defines
a Gibbs-type measure with the forecast law as its reference measure, \(J_k\) as the
energy, and \(\mathcal Z_k\) as the partition function
\citep{georgii2011gibbs}. When \(\Psi_k\) is the likelihood, this measure coincides
with the Bayesian posterior. A general compatibility potential instead defines an
energy-based analysis that need not correspond to a normalized observation law.

When the observation mechanism is accessible only through forward simulation, a
differentiable surrogate energy \(E_\phi(x;y_k)\) may be learned from simulated
state--observation pairs and substituted for \(J_k\). This construction is developed
in \cref{sec:lf}.

In all these cases, the formulation reweights the full forecast law without selecting
a particular inverse branch. It can therefore represent multiple well-separated state
configurations that are compatible with the same observation.
\subsection{The optimal controlled flow}
\label{sec:cflow}

To realize the energy-tilted analysis law in \cref{eq:ida}, we introduce an auxiliary
transport time \(\tau\in[0,1]\) within each assimilation cycle \(k\). Let
\(\mbP\) denote the path law of the base diffusion
\begin{equation}
\md Z_\tau
=
f(Z_\tau,\tau)\md\tau+\sigma(\tau)\md W_\tau,
\qquad
Z_0\sim\mu_0,
\quad
\tau\in[0,1],
\label{eq:osde}
\end{equation}
whose terminal law is the forecast law,
\(\mathrm{Law}_{\mbP}(Z_1)=\pf_k\). The base drift generally depends on the
forecast law at cycle \(k\), although this dependence is suppressed in the notation.
We then steer the base flow toward the analysis law \(\pa_k\) through an
observation-dependent drift \(u\),
\begin{equation}
\md Z_\tau
=
\left[
f(Z_\tau,\tau)+u(Z_\tau,\tau;y_k)
\right]\md\tau
+\sigma(\tau)\md W_\tau,
\qquad
Z_0\sim\mu_0,
\quad
\tau\in[0,1].
\label{eq:csde}
\end{equation}
We denote the path law of \cref{eq:csde} by \(\mbQ^u\). This representation also
includes deterministic flows when \(\sigma=0\). For the stochastic optimal-control
construction below, we assume that \(\sigma(\tau)>0\).

The control is chosen by minimizing the free-energy objective
\begin{equation}
\mathcal J_k(u)
\coloneqq
\mbE_{\mbQ^u}
\left[
\frac12
\int_0^1
\norm{\sigma(\tau)^{-1}u(Z_\tau,\tau;y_k)}^2
\md\tau
+
J_k(Z_1;y_k)
\right].
\label{eq:soc}
\end{equation}
The running term penalizes deviation from the base path law, whereas the terminal
term favors states compatible with the observation. Under the corresponding
Girsanov change of measure, the running cost equals
\(\KLdiv{\mbQ^u}{\mbP}\) for path laws sharing the initial distribution
\(\mu_0\).

\begin{assumption}[Regularity]
\label{ass:reg}
The base drift \(f(\cdot,\tau)\) is measurable in \(\tau\), locally Lipschitz in
\(z\), and of at most linear growth, and the noise schedule \(\sigma\) is continuous
and strictly positive. An admissible control \(u\) is progressively measurable,
makes \cref{eq:csde} well posed, and satisfies
\begin{equation}
\mbE_{\mbQ^u}
\int_0^1
\norm{\sigma(\tau)^{-1}u(Z_\tau,\tau;y_k)}^2
\md\tau
<\infty,
\end{equation}
together with the conditions required for the Girsanov change of measure. Moreover,
\begin{equation}
h(z,\tau;y_k)
\coloneqq
\mbE_{\mbP}
\left[
\me^{-J_k(Z_1;y_k)}
\mid Z_\tau=z
\right]
\label{eq:h-function}
\end{equation}
is strictly positive and continuously differentiable in \(z\) for \(\tau<1\).
\end{assumption}

These are standard conditions for the Doob \(h\)-transform associated with a
terminal energy \citep{follmer1988,leonard2014survey}. For the VP construction
below, the Gaussian transition kernel provides interior regularization under suitable
integrability conditions.

\begin{proposition}[Optimal control of the analysis step, \cite{follmer1988,dai1991stochastic,boue1998variational}]
\label{prop:optctrl}
Under \cref{ass:reg}, the minimizer of \cref{eq:soc} is the Markov control
\begin{equation}
u^\star(z,\tau;y_k)
=
\sigma^2(\tau)\nabla_z\log h(z,\tau;y_k)
=
\sigma^2(\tau)\nabla_z\log
\mbE_{\mbP}
\left[
\me^{-J_k(Z_1;y_k)}
\,\middle|\,
Z_\tau=z
\right].
\label{eq:ustar}
\end{equation}
If the base flow is memoryless 
\citep{domingoenrich2025adjoint}, the
controlled terminal law satisfies
\begin{equation}
\mathrm{Law}_{\mbQ^\star}(Z_1)=\pa_k.
\label{eq:optimal-terminal}
\end{equation}
\end{proposition}

The memoryless condition is essential because otherwise the normalization of the
terminal tilt generally depends on \(Z_0\), and the controlled terminal law need not
coincide with the global analysis law in \cref{eq:ida}.

\begin{definition}[Controlled-flow filter]
\label{def:cff}
At assimilation cycle \(k\), a controlled-flow analysis is specified by
\((\mu_0,f,\sigma,u)\), where the uncontrolled flow has terminal law \(\pf_k\)
and the observation-dependent drift \(u\) defines the guided dynamics
\cref{eq:csde}. Samples from its terminal law form the analysis ensemble.
The controlled-flow analysis is exact if
\(\mathrm{Law}(Z_1)=\pa_k\).
A controlled-flow filter applies such an analysis recursively after each forecast
step.
\end{definition}

Several existing flow-based ensemble filters admit the decomposition in
\cref{eq:csde}, including the score-based EnSF
\citep{bao2024ensf}, the flow-matching EnFF \citep{enff}, and the
Schr\"odinger-bridge EnSBF \citep{Sun2025EnSBF}. They differ in their initial
law, base dynamics, and observation-dependent drift. This common representation
does not imply that all these methods solve the stochastic optimal-control problem
\cref{eq:soc}. Rather, it isolates how each method constructs its forecast transport
and incorporates the observation. The resulting decompositions are summarized in
\cref{tab:families}.%, with derivations deferred to \cref{sec:existing-methods}.

\begin{table}[t]
\centering
\renewcommand{\arraystretch}{1.2}%
\resizebox{\textwidth}{!}{%
\begin{tabular}{@{}llllll@{}}
\toprule
family & \(\mu_0\) & \(\sigma\) & base drift \(f(z,\tau)\) & control \(u(z,\tau;y_k)\) & requirements \\
\midrule
EnSF \citep{bao2024ensf} & \(\mN(0,I_n)\) & \cref{eq:memoryless} & \cref{eq:EnCF-f}  & \(\sigma^2(\tau)\lambda(\tau)\nabla J_k\) & \(\nabla J_k\) at each step \(Z_\tau^{(j)}\)\\
EnFF-OT \citep{enff} & \(\mN(0,I_n)\) & 0 & \(\sum_{j=1}^N w_j\frac{Z_1^{(j)}-(1-\sigma_{\min})z}{1-(1-\sigma_{\min})\tau}\) & \(\lambda\nabla J_k\) & \(\nabla J_k\) at each step \(Z_\tau^{(j)}\)\\
EnFF-F2P \citep{enff} & \(\pa_{k-1}\) & 0 & \(\sum_{j=1}^N w_j(Z_1^{(j)}-Z_0^{(j)})\) & \(\lambda\nabla J_k\) & \(\nabla J_k\) at each step \(Z_\tau^{(j)}\)\\
EnSBF \cite{Sun2025EnSBF} & \(\delta_0\) & \(\sigma_0\)  & \(\sum_{j=1}^N w_j\frac{Z_1^{(j)}-z}{1-\tau}\) &
\(\sum_{j=1}^N w_j'\frac{Z_1^{(j)}-z}{1-\tau}\) & \(J_k\) at each step \(Z_\tau^{(j)}\)\\
\midrule
\textbf{EnCF (ours)} & \(\mN(0,I_n)\) & \cref{eq:memoryless} & \cref{eq:EnCF-f} & $\sigma^2(\tau)a_\theta$ & \(\nabla J_k\) at terminal \(Z_1^{(j)}\)\\
\textbf{EnCF-LF (ours)} & \(\mN(0,I_n)\) & \cref{eq:memoryless} & \cref{eq:EnCF-f} & $\sigma^2(\tau)a_\theta$ & \(y_k\sim p(\cdot\mid x_k)\) sampler\\
\bottomrule
\end{tabular}}
\caption{Flow-based ensemble filters expressed through the common controlled-flow
decomposition \cref{eq:csde}. Only EnCF and EnCF-LF construct the observation-dependent
drift by solving \cref{eq:soc} through adjoint matching.}
\label{tab:families}
\end{table}

\subsection{The adjoint-matching control}
\label{sec:am}
We now construct EnCF by specifying a memoryless base flow and learning its
observation-dependent control. Let \(\alpha_\tau\) and \(\beta_\tau\) be interpolation
coefficients satisfying \((\alpha_0,\beta_0)=(0,1)\) and \((\alpha_1,\beta_1)=(1,0)\)
so that the Gaussian kernel \(q_{\tau\mid0}(z\mid x)=\mN(z;\alpha_\tau x,\beta_\tau^2 I_n)\)
interpolates from the reference law
\(\mu_0=\mN(0,I_n)\) to a point mass at \(x\). The associated reverse
variance-preserving (VP) diffusion \citep{song2021sde,Vincent2011} has coefficients
\begin{equation}
b(\tau)
=
\frac{\dot\alpha_\tau}{\alpha_\tau},
\qquad
\sigma^2(\tau)
=
2b(\tau)\beta_\tau^2
-
\frac{\md\beta_\tau^2}{\md\tau}
=
\frac{
2\beta_\tau
\bigl(
\dot\alpha_\tau\beta_\tau
-
\dot\beta_\tau\alpha_\tau
\bigr)
}{
\alpha_\tau
}.
\label{eq:memoryless}
\end{equation}
This choice makes the endpoints of the base process independent,
\(Z_0\perp Z_1\), and hence satisfies the memoryless condition required in
\cref{prop:optctrl}.

At assimilation cycle \(k\), let
\(\{\hat Z_k^{(j)}\}_{j=1}^N\) denote the forecast ensemble approximating
\(\pf_k\). Then we may obtain an empirical base drift by Monte-Carlo approximation as
\begin{equation}
\hat f_k(z,\tau)
=
b(\tau)z
+
\sigma^2(\tau)
\sum_{j=1}^N
w_j(z,\tau)
\frac{
\alpha_\tau\hat Z_k^{(j)}-z
}{
\beta_\tau^2
},\quad w_j(z,\tau)=\frac{q_{\tau\mid0}(z\mid\hat Z_k^{(j)})}{\sum_{j'=1}^Nq_{\tau\mid0}(z\mid\hat Z_k^{(j')})}.
\label{eq:EnCF-f}
\end{equation}
The corresponding population drift transports \(\mu_0\) to \(\pf_k\), while
\cref{eq:EnCF-f} replaces the forecast law by its ensemble approximation~\citep{gao2024flow}.
% The reverse-time derivation and implementation details are provided in \cref{app:base}.

With the base flow fixed, it remains to approximate the optimal control in
\cref{eq:ustar}. We parameterize the observation-dependent drift as
\begin{equation}
u_\theta(z,\tau;y_k)
=
\sigma^2(\tau)a_\theta(z,\tau;y_k),
\label{eq:control}
\end{equation}
where \(a_\theta\) approximates the logarithmic gradient in
\cref{eq:ustar}. In the experiments, \(a_\theta\) is represented by a
convolutional neural network, although the construction is independent of this
choice.

Direct minimization of \(\mathcal J_k(u_\theta)\) in \cref{eq:soc} requires
differentiating through the sampled controlled trajectories. Adjoint matching \citep{domingoenrich2025adjoint,domingoenrich2026smp}
instead constructs a least-squares target by solving a lean adjoint equation
backward along each sampled path
\(Z\sim\mbQ^{u_\theta}\), with the path treated as stop-gradient. Because the optimal terminal control in \cref{eq:ustar} is known explicitly at \(\tau=1\) as \(u^\star(z,1;y_k)=-\sigma^2(1)\nabla_z J_k(z;y_k)\), we start from this terminal information and define the lean adjoint by
\begin{equation}
\widetilde a(1;Z)
=
\nabla_z J_k(Z_1;y_k),
\qquad
\frac{\md\widetilde a(\tau;Z)}{\md\tau}
=
-
\nabla_z f(Z_\tau,\tau)^\trans
\widetilde a(\tau;Z).
\label{eq:adjoint}
\end{equation}
The control is then trained using the adjoint-matching objective
\begin{equation}
\mathcal L_{\mathrm{AM}}(\theta;Z)
=
\frac12
\int_0^1
\sigma^2(\tau)
\norm{
a_\theta(Z_\tau,\tau;y_k)
+
\widetilde a(\tau;Z)
}^2
\md\tau.
\label{eq:am}
\end{equation}
Thus \(a_\theta\) is regressed onto
\(-\widetilde a\), or equivalently, the controlled drift \(u_\theta\) is
regressed onto \(-\sigma^2\widetilde a\). The observation enters the backward equation
only through the terminal seed
\(\nabla_zJ_k(Z_1;y_k)\). Moreover, the stop-gradient construction avoids
retaining an automatic-differentiation graph through the complete controlled
rollout. 
% Details of the paired-sample implementation are given in
% \cref{app:am-impl}.
% \todo{fix this}
\begin{proposition}[Optimality of adjoint matching
{\cite{domingoenrich2025adjoint,domingoenrich2026smp}}]
\label{prop:am}
Suppose that the base drift is exact and the diffusion coefficient depends only
on time. Under the regularity and uniqueness conditions of
\cite{domingoenrich2025adjoint,domingoenrich2026smp}, the optimal control
\(u^\star\) in \cref{prop:optctrl} is the unique critical point, in the
nonparametric control space, of the expected on-policy adjoint-matching objective \(\mbE_{Z\sim\mbQ^{u_\theta}}\left[\mathcal L_{\mathrm{AM}}(\theta;Z)\right]\),
where the sampled trajectories and adjoint targets are treated as stop-gradient.
\end{proposition}

\begin{algorithm}[t]
\caption{Ensemble Controlled-flow Filtering (\textbf{EnCF}/\textbf{EnCF-LF}).
}
\label{alg:encf}
\begin{algorithmic}[1]
\Require dynamics \(\{\mathcal M_k\}_k\); controlled-flow parameters \((\mu_0,f,\sigma)\). For EnCF, we also require the terminal energy \(J(\cdot;y_k)\) for each \(k\).
\State Initialize the control net \(a_\theta\) as zero.
\State Set up background ensemble members \(\{x_0^{(j)}\}_{j=1}^N\);\Comment{\(x_0^{(j)}\sim\pa_0\)}
\For{\(k=1,2,\cdots\)}\Comment{Assimilation cycle}
    \State \(\hat Z_k^{(j)}\gets \mM_k(x_{k-1}^{(j)},\eta^{(j)})\) for each \(j\); \Comment{forecast step, \(\hat Z_k^{(j)}\sim\pf_k\)}
    % \Comment{$\mathrm{Law}(Z_k)\approx\pf_k$}
\If{we use the EnCF-LF model}\Comment{EnCF-LF branch}
    \State Sample \(\hat y^{(j)}\sim p(\cdot\mid\hat Z_k^{(j)})\) for each \(j\);
    \State Learn \(E_\phi(x,y)\) with \(\{(\hat Z_k^{(j)},\hat y^{(j)})\}_{j=1}^N\);\Comment{Inner training loop}
    \State Set \(J\gets E_\phi\);
% \Else\Comment{EnCF branch}
%     \State Set $J$ as the ground truth;
\EndIf
    \Repeat%{$e=1,\cdots,E_k$}
      \Comment{Train with adjoint-matching}
    \State Initialize and evolve \(Z_\tau\) via \cref{eq:csde} with \(u_\theta=\sigma^2 a_\theta\) under stop-gradient;
    \State Set \(\tilde a(1)\gets\nabla_z J(Z_1;y_k)\) and integrate \(\dot{\tilde a}=-\tilde a^\trans\nabla f\) backward;
    \State Do an \(N_\mathrm{reg}\)-step optimization to minimize the regression loss
    \[\sum_\tau\sigma(\tau)^2\norm{a_\theta(Z_\tau,\tau;y_k)+\tilde a(\tau)}^2;\]
    \Until{\(E_k\) iterations end;}
    \State Initialize and evolve \(Z_\tau^{(j)}\) via \cref{eq:csde} with \(u_\theta=\sigma^2 a_\theta\) for each \(j\);
    \State \(x_k^{(j)}=Z_1^{(j)}\) for each \(j\).\Comment{analysis step, \(x_k^{(j)}\sim\pa_k\)}
\EndFor
\end{algorithmic}
\end{algorithm}

\subsection{The likelihood-free surrogate energy}\label{sec:lf}
When the observation law is unavailable but observation samples can be generated, we replace \(J_k\) by a
surrogate energy \(E_\phi\). We use two parameterizations according to the
available observation structure.

When a meaningful observation mapping exists, Gaussian-LF learns a pointwise
conditional model with mean \(m_\phi(x)\) and scale \(s_\phi\). Its energy is
\begin{equation}
    E_\phi(x;y)
    =
    \frac{1}{2}
    \norm{\frac{m_\phi(x)-y}{s_\phi}}^2,
\end{equation}
up to terms independent of \(x\). The parameters are fitted online by Gaussian
maximum likelihood on simulator pairs \((x^{(j)},y^{(j)})\). This model is
appropriate for smooth additive observations.
For more general observations such as multimodal or implicit conditional laws, the energy is directly parameterized as a scalar network. The network is trained by noise-contrastive
estimation, using matched simulator pairs as positives and mismatched
observations together with perturbed states as negatives. At each assimilation
cycle, \(E_\phi\) is first trained with the control fixed and is then held fixed
during adjoint matching. In both variants, the controlled-flow solver uses the
surrogate only through the terminal gradient
\(\nabla_x E_\phi(Z_1;y_k)\).

\section{Theory}
\label{sec:theory}
Building on the controlled-flow formulation developed in \cref{sec:method}, this
section establishes structural a priori guarantees for the resulting filtering
scheme along three complementary directions. We first study \emph{consistency} in
the ideal limit, showing that an exact base flow together with the optimal control
recovers the exact filtering distribution. We then characterize \emph{local
accuracy} by decomposing the one-step filtering error into refinable contributions
from the empirical base flow, control approximation, surrogate energy, and time
discretization. Finally, using the perturbation framework for Markov filters
\citep{leglandoudjane2004}, we establish \emph{stability} of the approximate
recursion by showing that, whenever the exact filter is stable, these locally
injected errors remain uniformly controlled rather than accumulating over
assimilation cycles.

\subsection{Consistency in the ideal limit}
\label{sec:consistency}
We first consider the idealized setting in which the forecast propagation, the
base flow, and the control approximation are exact. In this setting, the
controlled-flow construction recovers the target analysis law at every assimilation
cycle.
\begin{theorem}[Consistency of the controlled-flow filter]
\label{thm:consistency}
Suppose that the filter is initialized from the exact analysis law \(\pa_0\). For each cycle \(k\), assume that
\begin{itemize}
    \item the memoryless base flow \cref{eq:osde} transports \(\mu_0\) to \(\pf_k\) exactly.
    \item the adjoint-matching regression \cref{eq:am} reaches its on-policy fixed
    point.
\end{itemize}
Then the controlled flow terminates at \(\mathrm{Law}(Z_1)=\pa_k\) for each \(k\).
\end{theorem}

\begin{proof}
The result follows by induction over the assimilation cycles. The initial analysis
law is exact by assumption. 
For each cycle \(k\), by \cref{prop:am}, the on-policy fixed
point of the adjoint-matching regression recovers \(u^\star\). Since the base flow is
memoryless, \cref{prop:optctrl} implies that the corresponding controlled terminal
law is \(\pa_k\). The claim therefore follows for all \(k\).
\end{proof}
In practice EnCF departs from this ideal by four separately measurable amounts, which the next
subsection isolates in \cref{lem:decomp} and bounds in \cref{thm:onestep}, a base-drift Monte-Carlo
error, the suboptimality of the learned control, the surrogate energy gap, and the
time-discretization of the controlled flow.

\subsection{Local accuracy and one-step error decomposition}
\label{sec:onestep}
We isolate four sources of one-step approximation error under two structural
assumptions. \Cref{lem:decomp} first decomposes the error into these contributions.
We then bound the terms that admit analytic control and assemble the resulting
a priori estimate in \cref{thm:onestep}.

Formally, 
starting from the previous analysis $\nu\coloneqq\pa_{k-1}$, 
the exact one-step filtering operator
\(\Phi_k\coloneqq \mathcal B_{k}\circ(\mM_k)_\sharp\) consists of a predictor $(\mM_k)_\sharp:\pa_{k-1}\mapsto\pf_k$ and an exact analysis operator
\(\mathcal B_{k}(\nu)(\md x)\propto \me^{-J_k(x;y_k)}\nu(\md x)\).
In this subsection, we study how the practical update operator $\hat\Phi_k$ differs from the exact $\Phi_k$.

\begin{assumption}[Regularity]\label{ass:reg-thm}
The functions \(\hat f(\cdot,\tau)\) and \(u_\theta(\cdot,\tau;y_k)\) belong to
\(C^4(\mathbb R^n;\mathbb R^n)\), have bounded derivatives of orders
one through four, and satisfy a linear growth condition, uniformly in
\(\tau\in[0,1]\). 
The scalar diffusion coefficient
\(\sigma\in C_b^2([0,1])\) satisfies
\begin{equation}
    0<\sigma_{\min}\le\sigma(\tau)\le\sigma_{\max}<+\infty,\qquad\tau\in[0,1]
\end{equation}
for some constants $\sigma_{\min}$ and $\sigma_{\max}$.
\end{assumption}

\begin{assumption}[Girsanov Integrability]\label{ass:integ}
Let \(\hat f\) be the Monte-Carlo base drift \cref{eq:EnCF-f} and let
\begin{equation}
    \hat b=\hat f+u_\theta,
    \qquad
    b^\star=f+u^\star
\end{equation}
denote the practical and ideal controlled drifts, respectively.
We assume that the corresponding flows are well posed and that the drift discrepancy satisfies
\begin{equation}
    \mbE_{\widehat{\mathbb Q}}
    \int_0^1
    \frac1{\sigma(\tau)^2}\norm{\hat b(z,\tau)-b^\star(z,\tau)}^2
    \md\tau
    <\infty.
\end{equation}
Moreover, the exponential local martingale associated with this change of drift is
a true martingale, so that the Girsanov change of measure between the two path laws
is valid. A sufficient condition is the corresponding Novikov condition.
\end{assumption}
Fix an assimilation cycle \(k\). Let \(\widehat{\mathbb Q}\) denote the path law of
the practical continuous-time flow
\begin{equation}
    \mathrm dZ_\tau
    =
    \bigl[
        \hat f(Z_\tau,\tau)
        +
        u_\theta(Z_\tau,\tau;y_k)
    \bigr]\mathrm d\tau
    +
    \sigma(\tau)\md W_\tau,
    \qquad Z_0\sim\mu_0,
\end{equation}
and let \(\mathbb Q^\star\) denote the path law of the ideal flow

\begin{equation}
    \mathrm dZ_\tau
    =
    \bigl[
        f(Z_\tau,\tau)
        +
        u^\star(Z_\tau,\tau;y_k)
    \bigr]\mathrm d\tau
    +
    \sigma(\tau)\md W_\tau,
    \qquad Z_0\sim\mu_0.
\end{equation}
We write \(\widehat{\mathbb Q}_1\) and \(\mathbb Q^\star_1\) for their terminal
marginals, and \(\widehat{\mathbb Q}^{\mathrm{EM}}_1\) for the terminal law produced
by the Euler--Maruyama discretization of the practical flow. Thus, the practical terminal law \(\hat{\mbQ}^{\mathrm{EM}}_1=\hpa_k\coloneqq\hat\Phi_k(\pa_{k-1})\) and the ideal terminal law \(\mbQ^\star_1=\pa_k=\Phi_k(\pa_{k-1})\).

\begin{lemma}[One-step error decomposition]\label{lem:decomp}
Under \cref{ass:reg-thm,ass:integ}, the local one-step error admits the decomposition
\begin{equation}\label{eq:decomp}
  \begin{aligned}
    \TV{\hpa_k}{\pa_k}
    &\le\underbrace{\norm{\hat f-f}_{\hat\mbQ}}_{\text{base-drift error}}+\underbrace{
    \norm{u_\theta-u^\star_{E_\phi}}_{\hat{\mbQ}}
    }_{\text{control suboptimality}}+\underbrace{\norm{u^\star_{E_\phi}-u^\star}_{\hat\mbQ}}_{\text{energy mismatch}}+\underbrace{\TV{\hat\mbQ^{\mathrm{EM}}_1}{\hat\mbQ_1}}_{\text{discretization}}
  \end{aligned}
\end{equation}
for the total-variation (TV) discrepancy, where we introduce a path-space seminorm
\begin{equation}\label{eq:path-norm}
    \norm{g}_{\hat\mbQ}=\frac12\sqrt{\mbE_{\hat\mbQ}\int_0^1\frac1{\sigma(\tau)^2}\norm{g(Z_\tau,\tau)}^2\md\tau}
\end{equation}
for any integrable \(g=g(Z_\tau,\tau)\), and $u_{E_\phi}^\star$ stands for the optimal control corresponding to the energy $E_\phi$.
\end{lemma}

\begin{proof}
Since the two continuous-time flows have the same initial law and diffusion
coefficient, Girsanov's theorem gives
\begin{equation}\label{eq:KL-hat-Q-Q-star}
\KLdiv{\hat{\mathbb Q}}{\mathbb Q^\star}
=\frac12\mbE_{\hat\mbQ}\int_0^1\frac1{\sigma(\tau)^2}\norm{\bigl(\hat f-f\bigr)(Z_\tau,\tau)+(u_\theta-u^\star)(Z_\tau,\tau;y_k)}^2\md\tau,
\end{equation}
which is finite by \cref{ass:integ}. Pinsker's inequality and data processing from path laws to their terminal marginals imply
\begin{equation}\label{eq:TV-KL}
    \TV{\hat\mbQ_1}{\mbQ^\star_1}\le\sqrt{\frac12\KLdiv{\hat\mbQ_1}{\mbQ_1^\star}}\le\sqrt{\frac12\KLdiv{\hat\mbQ}{\mbQ^\star}}.
\end{equation}
Sampling \(\hat\mbQ\) uses its Euler--Maruyama discretization \(\hat\mbQ^{\mathrm{EM}}\), so the triangle inequality gives
\begin{equation}\label{eq:TV-EM-err}
    \TV{\hpa_k}{\pa_k}=\TV{\hat\mbQ^{\mathrm{EM}}_1}{\mbQ^\star_1}\le \TV{\hat\mbQ^{\mathrm{EM}}_1}{\hat\mbQ_1}+\TV{\hat\mbQ_1}{\mbQ^\star_1}.
\end{equation}
Finally, combining \cref{eq:KL-hat-Q-Q-star,eq:TV-KL,eq:TV-EM-err} and applying the triangle inequality to the path seminorm \cref{eq:path-norm} yields \cref{eq:decomp}.
\end{proof}
We then discuss the four terms in turn. The decomposition identifies the numerical levers of the
method, the ensemble representation and base drift for the base-drift term, the adjoint matching optimization error for the control term, 
the surrogate gradient mismatch for the energy term, and
the number of integration steps \(N_\tau\) for the discretization term.
% Only some of these levers yield derived rates. 
These terms have different theoretical status. The
base-flow and control errors are retained as approximation residuals, whereas the
energy and discretization terms admit explicit analytic bounds under additional
regularity assumptions.

For the base-drift term, \cref{prop:mc-error} proves in
the appendix the self-normalized importance-sampling rate \(\mathcal O(N^{-1/2})\), with a constant
governed by the \(\chi^2\) dispersion of the base weights. 
% For the deployed single-sample base
% (\cref{rmk:onesample}), this weighted score-estimation variance is removed. The remaining contribution
% is the unavoidable ensemble representation error of \(\pf_k\).
The control term measures the optimization errors between the learned control and the nonparametric fixed point induced by the
on-policy adjoint-matching regression as detailed in \cref{prop:am}. We retain
this term as a residual because a general convergence rate would additionally
require approximation estimates for the control class and quantitative guarantees
for the optimization procedure, which is beyond our scope in this paper.

\begin{proposition}[Continuous dependence of the optimal control on the energy]\label{prop:energy}
Let \(J_1,J_2\in C^1(\mathbb R^n)\) satisfy
\(\norm{\nabla J_1-\nabla J_2}_\infty<\infty\). Consider the homotopy \(J_r=(1-r)J_1+rJ_2\) and define the probability density
\(\rho_r^{z,\tau}\propto \tilde\rho_r^{z,\tau}\coloneqq\exp(-J_r)p_{1\mid\tau}(\cdot\mid z)\). Assume that
\begin{itemize}
    \item the score function \(\nabla\log p_{1\mid\tau}(z'\mid z)\) for the base flow has bounded variance
    \begin{equation}
        V\coloneqq\sup_{r,z,\tau}\Tr\Var_{\rho_{J_r}}\nabla\log p_{1\mid\tau}(\cdot\mid z)<+\infty;
    \end{equation}
    \item every \(\rho_r^{z,\tau}\) satisfies the Poincar\'e inequality
    \begin{equation}
        \Var_{\rho_r^{z,\tau}}(\varphi)
    \le
    C_P
    \mbE_{\rho_r^{z,\tau}}
    \norm{\nabla\varphi}^2,\qquad
    \forall\varphi\in H^1(\rho_r^{z,\tau})
    \end{equation}
    with a constant \(C_P\) independent of \(r\), \(z\), and \(\tau\).
\end{itemize}
    Then the optimal control \(u^\star\) depends continuously on the energy function \(J\). Formally, we have
    \begin{equation}
        \norm{u^\star_{J_1}-u^\star_{J_2}}_\infty\le \sigma(\tau)^2\sqrt{C_PV}\norm{\nabla J_1-\nabla J_2}_{\infty}.
    \end{equation}
    Taking \(J_1=E_\phi\) and \(J_2=J\) bounds the energy term of \cref{eq:decomp} by the energies' gradient mismatch alone. It follows immediately that
\begin{equation}
    \norm{u^\star_{E_\phi}-u^\star}_{\hat\mbQ}\le \frac{\sigma_{\max}}2\sqrt{C_PV}\norm{\nabla E_\phi-\nabla J}_{\infty}.
\end{equation}
\end{proposition}
\begin{proof}[Proof sketch]
{Write the optimal control \cref{eq:ustar} as a Gibbs average of the base-kernel score,
whose differentiation along the homotopy is exactly a covariance between \(J_1-J_2\) and the base score function.
Integrating in \(r\) and applying Cauchy--Schwarz, the bounded score variance \(V\), and the Poincar\'e
inequality with constant \(C_P\) gives the bound. The complete calculation
is provided in \cref{app:proof-energy}.}
\end{proof}

Two consequences make this the pivot of the implicit-energy and likelihood-free
design. First, although \(u_J^\star\) is nonlocal in \(J\), replacing \(J\) by a learned surrogate perturbs
the filtering map only through \(\norm{\nabla E_\phi-\nabla J}_\infty\) under the stated assumptions, so
energy learning by adjoint matching is stable in the control metric. Besides, the control is invariant to shifts
\(J\mapsto J+c\) in \(x\). EnCF-LF therefore need not estimate the tilt normalizer over the state, making the maximum-likelihood training reasonable.

\begin{proposition}[Time-discretization error,  \cite{ballytalay1996,Bally1996sde}]\label{prop:disc}
Let \(\hat\mbQ^{\mathrm{EM}}\) be the path law induced by the uniform-step Euler--Maruyama discretization of \(\hat\mbQ\) with \(N_\tau\) steps. Under \cref{ass:reg-thm}, the corresponding terminal laws satisfy
\begin{equation}
    \TV{\hat\mbQ^{\mathrm{EM}}_1}{\hat\mbQ_1}\le C_{\mathrm{EM}}N_\tau^{-1},
\end{equation}
where the constant \(C_{\mathrm{EM}}\) depends only on the regularity constants. The estimate
follows from the total-variation convergence of the Euler--Maruyama scheme.
\end{proposition}

Combining the estimate of each component in the decomposition \cref{eq:decomp}, we have the following aggregated result.

\begin{theorem}[One-step filtering error]
\label{thm:onestep}
Suppose that the one-step update is initialized from the exact analysis law
\(\pa_{k-1}\). Under
\cref{ass:reg-thm,ass:integ} and the assumptions of \cref{prop:energy}, the local one-step filtering error satisfies
\begin{equation}\label{eq:decomp-detailed}
    \TV{\hpa_k}{\pa_k}\le\norm{\hat f-f}_{\hat\mbQ}+\norm{u_\theta-u^\star_{E_\phi}}_{\hat\mbQ}+L_\mathrm{E}\norm{\nabla E_\phi-\nabla J}_{\infty}+C_\mathrm{EM}N_\tau^{-1},
\end{equation}
where $L_\mathrm{E}=\tfrac12{\sigma_{\max}}\sqrt{C_PV}$. \hfill\proofbox
\end{theorem}

% \begin{proof}
% Substitute the energy bound of \cref{prop:energy} and the discretization bound of \cref{prop:disc} into the decomposition \cref{eq:decomp}. The remaining drift and control terms are characterized by \cref{prop:mc-error,prop:am-err}. \todo{fix this} {The pointwise bounds of \cref{prop:mc-error,prop:energy} pass to the path semi-norm \(\norm{\cdot}_{\hat\mbQ}\) by integrating over \([0,1]\), uniformly integrable under \cref{ass:integ}.}
% \end{proof}

\Cref{thm:onestep} separates the local filtering error according to the four
approximations made by the practical method. The first term depends on the
accuracy with which the empirical base dynamics represent the forecast transport.
The second reflects the expressiveness and optimization accuracy of the learned
control. The third is governed by the \(C^1\)-accuracy of the surrogate energy,
while the fourth decreases as the transport-time discretization is refined. The
bound is therefore structural and a priori rather than an online error certificate,
since the exact base drift \(f\) and optimal control \(u^\star\) are generally
unavailable.
The constant \(L_\mathrm{E}\) also indicates that surrogate-energy errors may be amplified
when the tilted bridge measures become poorly conditioned, as reflected by large
Poincar\'e or score-variance constants. This is consistent with the intuition that sharply concentrated observation models tend to be more challenging.

\subsection{Multi-step stability and error propagation}
The one-step estimate of \cref{thm:onestep} quantifies the local error introduced
during a single assimilation cycle. To propagate this estimate over multiple
cycles, we compare the practical recursion with the exact filtering recursion
started from the same intermediate distributions. The key requirement is stability
of the exact filtering operators with respect to perturbations of their input laws.

% \begin{definition}[Filtering operators]\label{def:filter-op}
% For the observation energy \(J_k(\cdot;y_k)\), define the exact analysis operator
% \(\mathcal B_{k}(\nu)(\md x)\propto \me^{-J_k(x;y_k)}\nu(\md x)\), then the exact one-step filtering operator is
% \(\Phi_k\coloneqq \mathcal B_{k}\circ(\mM_k)_\sharp\).
To study the accumulated filtering error, we write \(\Phi_{k\to l}\coloneqq\Phi_{l}\circ\cdots\circ\Phi_{k+1}\) for $k\le l$ and set \(\Phi_{k\to k}\) as the identity. Then we define
\begin{equation}\label{eq:dk-ek-def}
    d_k\coloneqq\TV{\hpa_k}{\pa_k},\quad e_k\coloneqq\TV{\hpa_k}{\Phi_k(\hpa_{k-1})}=\TV{\hat\Phi_k(\hpa_{k-1})}{\Phi_k(\hpa_{k-1})}
\end{equation}
as the accumulated filtering error and the one-step filtering error, respectively. Recall that we have actually provided an upper bound of \(e_k\) in \cref{thm:onestep}.
% When \(J_k=-\log\ell(\cdot;y_k)\) with the likelihood function $\ell(\cdot;y_k)$, the analysis operator \(\mathcal B_k\) reduces to the standard Bayes update. Meanwhile, we let $\hat\Phi_k$ be the practical filtering operator, accounting for the one-step error \cref{eq:decomp} from the exact $\Phi_k$. It follows directly from \cref{thm:onestep} that \cref{eq:decomp-detailed} actually estimates the operator difference between $\Phi_k$ and $\hat\Phi_k$ with respect to the TV distance.
% \end{definition}

\begin{assumption}[Exponential stability of the exact filter]
\label{ass:stab}
Let \(\mathcal P\) be a class of probability measures invariant under the exact
filtering operators. There exist constants \(C_{\mathrm{stab}}\ge1\) and
\(\varrho\in(0,1)\) such that
\begin{equation}
\label{eq:filter-stability}
    \TV{
        \Phi_{k\to l}(\nu)
    }{
        \Phi_{k\to l}(\nu')
    }
    \le
    C_{\mathrm{stab}}
    \varrho^{l-k}
    \TV{\nu}{\nu'}
\end{equation}
for every \(\nu,\nu'\in\mathcal P\) and all \(0\le k\le l\).
\end{assumption}
Estimates of the form \eqref{eq:filter-stability} are classical under suitable
uniform mixing conditions on the forecast kernels and upper and lower bounds on
the observation potentials
\citep{atarzeitouni1997,leglandoudjane2004}. 
Note that \cref{ass:stab} is a structural property of the exact filtering recursion independent of the choice of numerical implementation.

\begin{theorem}[Multi-step filtering error]\label{thm:multistep}
Suppose that \cref{ass:stab} holds and that all intermediate exact and practical
laws belong to \(\mathcal P\). Then the accumulated filtering error satisfies
\begin{equation}\label{eq:dl-upper-bound}
    d_l
    \le
    C_{\mathrm{stab}}\varrho^l d_0
    +
    \frac{C_{\mathrm{stab}}}{1-\varrho}
    \max_{1\le k\le l}e_k.
\end{equation}
Moreover, each one-step error \(e_k\) is bounded from above by \cref{thm:onestep} under certain assumptions.
In particular, if the filter is initialized exactly as $d_0=0$, then the accumulated filtering error $d_k$ is bounded by the maximum of one-step errors up to a constant multiplier.
\end{theorem}

\begin{proof}
The complete proof is deferred to \cref{app:proof-multistep}.
\end{proof}

\Cref{thm:multistep} does not assert that the practical filtering map is itself
contractive. Rather, it shows that local perturbations of a stable exact filtering
recursion are geometrically attenuated as they propagate forward in time. Thus,
uniformly bounded local errors produce a uniformly bounded global error, with
amplification factor \(C_{\mathrm{stab}}/(1-\varrho)\).

The one-step estimate \cref{thm:onestep} remains an a priori structural bound because its base-flow
and control terms involve the unavailable exact drift and optimal control.
\Cref{thm:multistep} complements this local estimate by describing how such
errors propagate across assimilation cycles. When \(\varrho\) is close to one,
the exact filter forgets perturbations slowly and the resulting multi-step bound
becomes correspondingly less informative.

% \begin{corollary}[Likelihood-free variant]
% \label{cor:lf}
% The Adjoint-Matching machinery touches \(J\) only through \(\nabla J\) evaluated at samples and detached
% from the regression graph (the adjoint seed and the control-net feature). Replacing a known differentiable
% energy \(J\) by a learned conditional energy \(E_\phi=-\log p_\phi(y\mid x)\), fitted from simulator samples,
% therefore leaves \cref{prop:optctrl} and \cref{prop:am} intact with \(J\to E_\phi\). If
% \(E_\phi\to J\) in \(C^1\) on the support of the controlled paths, the likelihood-free fixed point
% \emph{converges to} the likelihood-based one. This is a \(C^1\)-stability statement rather than exact
% coincidence, since \(J\) enters the optimum nonlocally through \(\E[e^{-J}\mid Z_\tau{=}\cdot]\), with the
% residual being the energy term of \cref{prop:energy}. The observation operator's analytic form,
% likelihood formula, and derivatives are never used by the controlled-flow solver once the surrogate
% energy is fitted.
% \end{corollary}

\section{Experiments}
\label{sec:experiments}
This section evaluates the proposed EnCF and its likelihood-free variant EnCF-LF
against classical ensemble filters and recently developed flow-based filters. The
experiments are organized according to the structure of the observation model. We
first consider a smooth additive-Gaussian setting as a sanity check
(\cref{sec:conventional}), where classical Kalman-type filters are expected to perform
well. We then examine a non-Gaussian additive noise model
(\cref{sec:nongaussian}), which progressively moves the posterior away from a
Gaussian approximation. Finally, we focus on the observation regimes targeted by
the proposed formulation: a many-to-one sign-blind sensor that induces a genuinely
multimodal posterior (\cref{sec:sign-blind}), and an implicit, non-additive
observation relation for which the likelihood-free variant uses simulator access only
(\cref{sec:implicit}). The computational cost of learning the controlled flow is
reported separately in \cref{sec:cost}.

\subsection{Filter configurations}\label{sec:filter-configurations}
The proposed filters, EnCF and EnCF-LF, are implemented following
\cref{alg:encf}. 
The control \(a_\theta\) is represented by a convolutional neural network. The state
\(z\) and the observation \(y\) are passed through separate convolutional encoders,
and the resulting feature maps are concatenated and decoded by a merge convolution
to produce the control. The network is initialized at zero. In EnCF, the observation
energy enters the adjoint-matching regression only through the terminal condition
\(\nabla J_k(Z_1;y_k)\), as described in \cref{alg:encf}. Full implementation details
and hyperparameters are provided in \cref{app:experiment}.

At each assimilation step \(k\), the control network of our EnCF is trained from scratch. We therefore use
more optimization in the early cycles and gradually reduce the training budget. In
all experiments, the number of training epochs \(E_k\) follows a cosine schedule,
inspired by \cite{loshchilov2016sgdr}, decreasing from \(100\) to \(1\) over the first
\(20\) assimilation cycles. After this warm-up phase, each analysis step uses a single
training epoch. The number of regression steps is set as $\Nreg=16$.

We compare against two classes of ensemble filters. The classical baselines are the
EnKF \cite{evensen2009}, the LETKF \cite{hunt2007letkf}, and the bootstrap particle
filter \cite{vanleeuwen2019}. For EnKF and LETKF, covariance inflation and
Gaspari--Cohn localization are tuned separately for each benchmark; the tuning
protocol is reported in \cref{app:experiment}. The particle filter uses the exact likelihood
whenever it is available, but is included mainly as a reference method because it is
known to suffer from weight degeneracy in high-dimensional problems
\cite{snyder2008obstacles}.
The flow-based baselines are EnSF \cite{bao2024ensf}, EnFF \cite{enff}, and EnSBF
\cite{Sun2025EnSBF}. For EnFF, we report both the optimal-transport variant
(EnFF-OT) and the filtering-to-predictive variant (EnFF-F2P). All flow-based
baselines use the same number of transport-time steps \(N_\tau=100\), matching the
EnCF implementation.

\subsection{Metrics}
For each benchmark, we run the filter for \(K\) assimilation cycles and report
time-averaged quantities over a converged window \(\mathcal C\), defined as the last
\(10\%\) of the rollout. We use three complementary metrics as follows.
\paragraph{Root mean square error (RMSE)} We evaluate the accuracy of the ensemble mean
\(\bar x_k=N^{-1}\sum_{j=1}^N x_k^{(j)}\)
by computing the root mean square error against the reference state
\(x_k^\star\in\mathbb R^n\), defined by
\begin{equation}
\mathrm{RMSE}_k=\sqrt{\frac1n\sum_{i=1}^n\bigl(\bar x_{k,i}-x_{k,i}^\star\bigr)^2},
\qquad
\mathrm{RMSE}=\frac1{|\mathcal C|}\sum_{k\in\mathcal C}\mathrm{RMSE}_k,
\label{eq:rmse}
\end{equation}
where \(i\) indexes coordinates. Lower RMSE indicates a more accurate ensemble mean.

\paragraph{Continuous ranked probability score (CRPS)}
Since the analysis output is an ensemble distribution rather than only a point
estimate, we also report the multivariate energy-score form of the continuous ranked
probability score
\begin{equation}
\mathrm{CRPS}_k=\frac1N\sum_{j=1}^N\norm{x_k^{(j)}-x_k^\star}_2
-\frac1{2N^2}\sum_{j=1}^N\sum_{j'=1}^N\norm{x_k^{(j)}-x_k^{(j')}}_2.
\label{eq:crps}
\end{equation}
The reported CRPS is averaged over \(k\in\mathcal C\). The first term penalizes
distance from the reference state, whereas the second term rewards ensemble spread.
Thus CRPS evaluates the ensemble as a distribution, with lower values indicating a sharper and better calibrated ensemble.

\paragraph{Spread-to-RMSE ratio (SRR)}
We use
\begin{equation}
    \mathrm{Spread}_k=\sqrt{\frac1{nN}\sum_{j=1}^N\sum_{i=1}^n\bigl(x_{k,i}^{(j)}-\bar x_{k,i}\bigr)^2},
    \qquad\mathrm{SRR}_k=\frac{\mathrm{Spread}_k}{\mathrm{RMSE}_k},
\end{equation}
as a calibration diagnostic.
The final SRR is again averaged over \(k\in\mathcal C\). Values near one indicate
that the ensemble spread is consistent with the realized error. Values below one
indicate under-dispersion or ensemble collapse, whereas values above one indicate
over-dispersion.

All configurations are run with five independent random seeds. Tables report the
mean and standard deviation over seeds, and figures show the same variability using
error bars or shaded bands.

\subsection{A conventional assimilation task}\label{sec:conventional}
We begin with a standard setting in which the observation model is smooth, dense,
and corrupted by additive Gaussian noise. This experiment is intended as a sanity
check, and the purpose is to verify
that our controlled-flow formulation remains stable and competitive before moving
to the non-Gaussian, multimodal, and implicit observation models considered afterward.

We consider the Kuramoto--Sivashinsky (KS) dynamics
\begin{equation}
    u_t+uu_x+u_{xx}+u_{xxxx}=0
\end{equation}
on a periodic domain discretized into \(1024\) grid points. The state is observed
through the componentwise nonlinear sensor
\begin{equation}\label{eq:ks-arctan}
    y=\mH(x)+\varepsilon\in\mathbb R^{1024},\qquad\mH(x)=\arctan(x),\quad
    \varepsilon\sim\mathcal N(0,0.1^2 I_{1024}).
\end{equation}
We run \(1000\) assimilation cycles with ensemble size \(N=20\).

\begin{table}
\centering
\begin{tabular}{@{}lccc@{}}
\toprule
method  & RMSE (\(\downarrow\)) & CRPS (\(\downarrow\)) & SRR \\
\midrule
EnKF             & \(0.026\pm0.000\) & \(0.60\pm0.01\) & \(1.114\pm0.018\) \\
LETKF            & \(0.025\pm0.001\) & \(0.58\pm0.03\) & \(1.131\pm0.044\) \\
\midrule
EnSF                     & \(0.080\pm0.001\) & \(1.89\pm0.03\) & \(0.715\pm0.012\) \\
EnFF-OT                  & \(0.076\pm0.001\) & \(2.37\pm0.02\) & \(0.038\pm0.001\) \\
EnFF-F2P                 & \(0.185\pm0.001\) & \(5.93\pm0.02\) & \(0.002\pm0.000\) \\
EnSBF                    & \(1.760\pm0.016\) & \(54.1\pm0.49\) & \(0.062\pm0.001\) \\
\midrule
% EnCF (\(\Nreg=1\))          & \(0.056\pm0.001\) & \(1.56\pm0.02\) & \(2.07\pm0.05\)\\
EnCF         & \(0.056\pm0.002\) & \(1.52\pm0.03\) & \(2.02\pm0.07\)\\
% EnCF-LF (\(\Nreg=1\))       & \(0.055\pm0.002\) & \(1.37\pm0.02\) & \(1.54\pm0.07\)\\
EnCF-LF      & \(0.068\pm0.003\) & \(1.65\pm0.05\) & \(1.35\pm0.07\)\\
\bottomrule
\end{tabular}
\caption{KS dynamics with the smooth additive-Gaussian arctangent observation
in \cref{eq:ks-arctan}. All methods use \(N=20\) ensemble members. Values are
mean \(\pm\) standard deviation over 5 random seeds.}
\label{tab:ksmain}
\end{table}

The results in \cref{tab:ksmain} are consistent with the expected behavior in this
conventional regime. The tuned EnKF and LETKF achieve the lowest RMSE and CRPS,
confirming that classical covariance-based updates remain preferable when the
observation model is smooth, dense, and additive Gaussian. Among the flow-based
methods, EnCF gives the best RMSE and CRPS, while EnCF-LF remains stable despite
using a learned surrogate observation energy. Several existing flow-based baselines,
especially EnFF-F2P and EnSBF, show very small SRR values, indicating strong
under-dispersion or ensemble collapse.
This sanity check provides a baseline for the more challenging non-Gaussian,
multimodal, and implicit observation regimes considered in the following
subsections.

\subsection{Non-Gaussian additive noise}
\label{sec:nongaussian}
We next keep the KS dynamics and the arctangent observation operator in
\cref{eq:ks-arctan}, but replace the Gaussian observation noise by a symmetric
two-component mixture. Coordinate-wise, the observation model is
\begin{equation}
    y=\mH(x)+\varepsilon,\quad \mH(x)=\arctan(x),\qquad
    \varepsilon_i\sim\mathcal N(a\xi_i\sigma,\sigma^2),
    \quad \sigma=0.1,
    \label{eq:bimodal-noise}
\end{equation}
where the signs \(\xi_i\) are independent Rademacher variables.
The scalar parameter
\(a\ge0\) controls the separation of the two noise modes in units of \(\sigma\).
Thus \(a=0\) recovers the Gaussian setting in \cref{sec:conventional}, while larger
values of \(a\) produce increasingly separated likelihood modes.
\Cref{fig:nongauss}(a) and (b) show how the likelihood changes as \(a\) increases.
The noise density develops two separated
modes, and the corresponding observation energy \(J=-\log \ell(\cdot\mid y)\)
changes from a single well to a double-well landscape.

Note that for a fixed \(a\), the marginal noise distribution has mean zero and variance \(\sigma^2(1+a^2)\).
Therefore, in addition to the nominal EnKF using \(R_k=\sigma^2 I\) (defined in \cref{eq:enkf}), we also report a matched-variance EnKF using \(R_k=\sigma^2(1+a^2)I\). The latter is more reasonable since it uses the correct marginal variance of the mixture, although
it still replaces the bimodal likelihood by a single Gaussian approximation.
\cref{fig:nongauss}(c) reports the converged RMSE for various offsets \(a\).
For small \(a\), the likelihood remains close to Gaussian, and the Kalman update is
the most effective choice. As \(a\) increases, the two likelihood modes separate, and
a single Gaussian residual model becomes less representative of the posterior. The
matched-variance EnKF partly accounts for the larger noise scale, but it still cannot
represent the bimodal structure of the likelihood. In contrast, EnCF directly uses
the observation energy and remains comparatively insensitive to the mode separation.
The RMSE curves cross around \(a\approx4\), and beyond this point, EnCF gives lower
error than both EnKF variants. This experiment therefore isolates a transition from
a regime where Gaussian covariance updates are sufficient to one where the controlled-flow approach becomes advantageous.

\begin{figure}
\centering
\includegraphics[width=\linewidth]{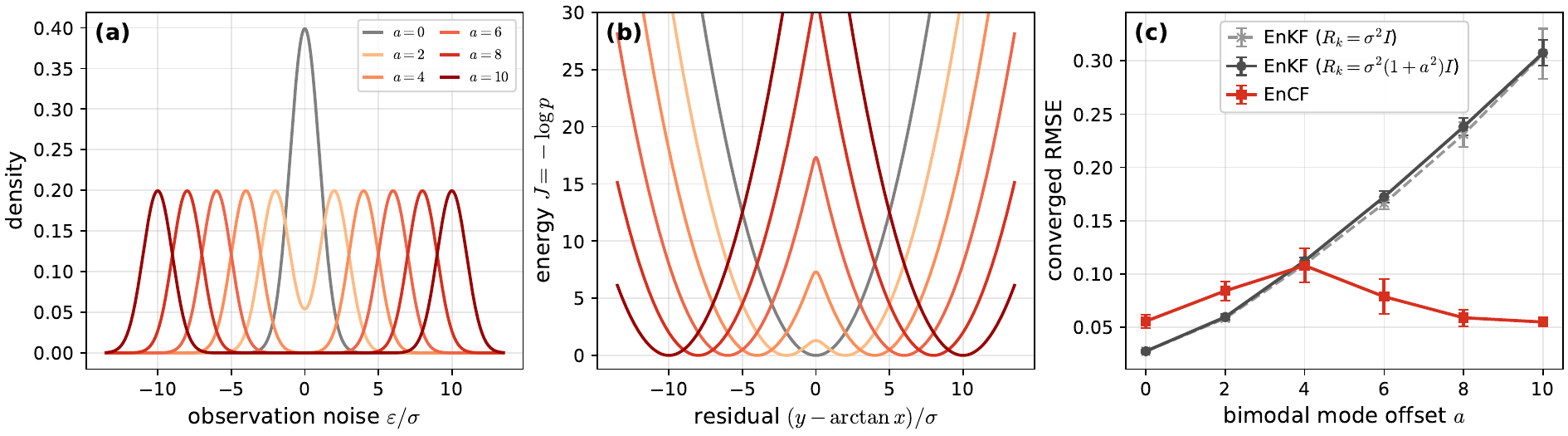}
\caption{KS dynamics with non-Gaussian additive noise ($N=20$, 5 seeds). \emph{(a)} Bimodal noise density in units of \(\sigma\) for various offsets \(a\). 
\emph{(b)} Induced observation energy \(J=-\log \ell\), which changes from a
single well at \(a=0\) to a double-well landscape as \(a\) increases.
\emph{(c)} Converged RMSE versus the mode offset \(a\) for the EnKF and the EnCF approaches.
}
\label{fig:nongauss}
\end{figure}

\subsection{Multimodal posteriors}
\label{sec:multimodal}
We now turn to observation models that induce genuinely multimodal analysis
distributions. The common difficulty in this section is that the state-to-observation
relation is many-to-one, so a single observation can be compatible with multiple
well-separated states. In such cases, a unimodal Gaussian update can collapse distinct
posterior modes, while methods relying directly on analytic likelihood gradients may
only exploit the local information exposed by the observation map. The two
benchmarks below isolate this issue in a controlled setting: the first uses an explicit
differentiable sign-blind sensor, and the second replaces the additive observation
model by an implicit, non-additive relation.

Both benchmarks in this section share the same double-well dynamics in \(\mathbb R^{20}\),
\begin{equation}
    \md x_t=-4 x_t(x_t^2-1)\md t+0.5\md W_t
\end{equation}
where the nonlinear operations are applied coordinate-wise and \(W_t\) is a
standard Brownian motion. 
The drift is induced by a double-well
potential \(V(x)\) with stable wells at \(x=\pm1\), as illustrated in
\cref{fig:dw-potential}. We use \(N=40\) ensemble members and assimilate every
\(20\) model steps for \(100\) assimilation cycles.

\begin{figure}
    \centering
    \begin{minipage}{.3\textwidth}
    \begin{tikzpicture}[scale=0.8]
      \draw[->] (-2, 0) -- (2, 0) node[right] {\(x\)};
      \draw[->] (0, -.5) -- (0, 2.5) node[above] {\(y=V(x)=x^4-2x^2\)};
      \draw[scale=0.5, domain=-1.8:1.8, smooth, variable=\x, red, thick] plot ({\x}, {\x*\x*\x*\x-2*\x*\x});
    \end{tikzpicture}
    \end{minipage}%
    \hfill\vline\hfill 
    \begin{minipage}{.65\textwidth}
        \includegraphics[width=\textwidth]{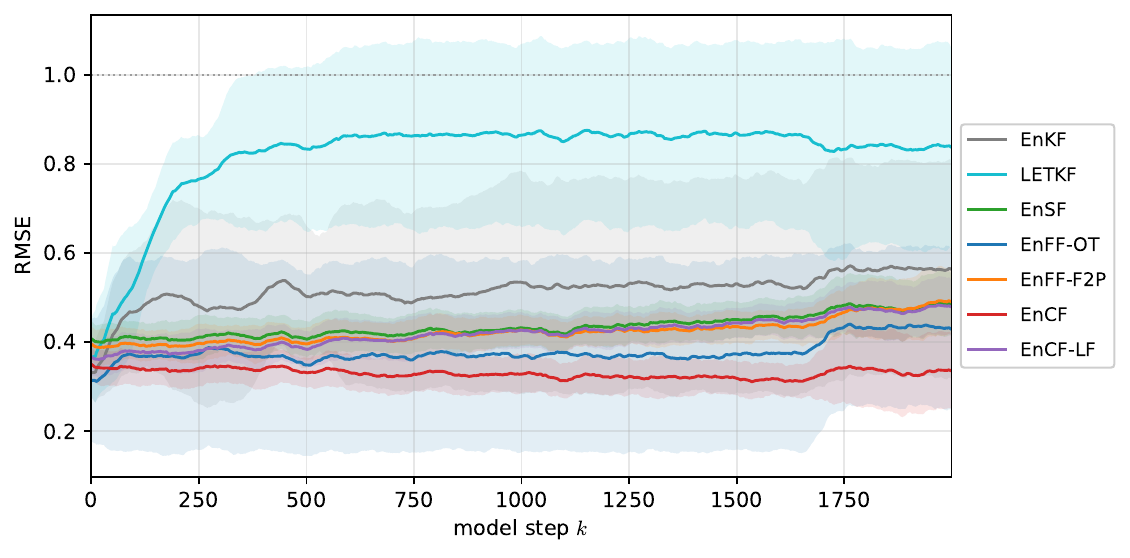}
    \end{minipage}
    \caption{Illustration of the double-well potential (left) and the evolution of assimilation error for each method (right).}
    \label{fig:dw-potential}
    
\end{figure}
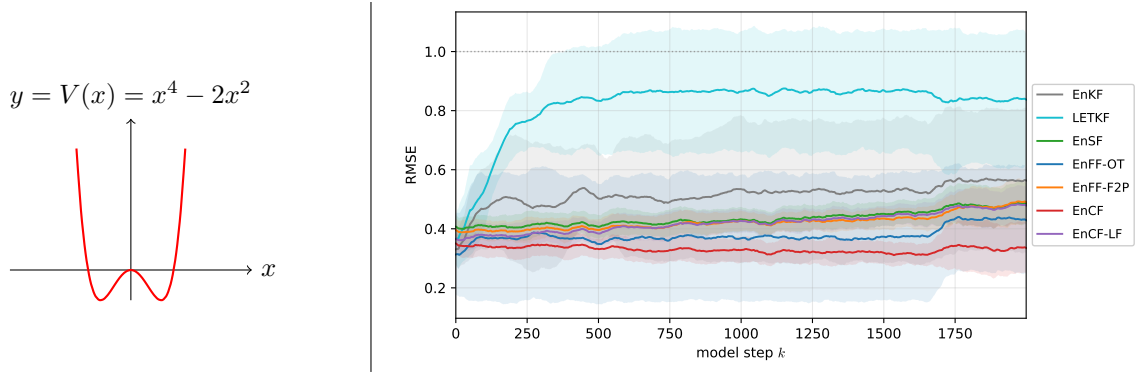

\subsubsection{Sign-blind magnitude observation}
\label{sec:sign-blind}
To start with, we consider an explicit many-to-one observation model. The observation consists of a dominant magnitude channel and a weak
linear sign channel
\begin{equation}
    y=\mH(x)+\varepsilon\in\mR^{40},\qquad\mH(x)=(x^2,cx)^{\trans},\quad
    \varepsilon\sim\mathcal N(0,0.1^2 I_{40}),
\end{equation}
where all operations are applied coordinate-wise. The first part is sign-blind, and thus we introduce a linear channel with slope \(c=0.02\) that carries the sign at effective noise \(0.1/0.02=5\) times the state scale, making the
problem hard but identifiable over time.
By this construction, each analysis step remains strongly multimodal.

The RMSE evolution in \cref{fig:dw-potential} shows that the Kalman-type filters
struggle in this regime even after tuning inflation and localization. The existing flow-based filters improve
over the Kalman baselines but remain limited by the artificial controls used to guide
the transport. In contrast, EnCF achieves the lowest error among the compared
methods.
This supports the intended role of the controlled-flow formulation: rather
than imposing a Gaussian update or relying on a prescribed likelihood-score control
along the path, it learns a transport induced by the terminal observation energy and
can better accommodate the two-well posterior structure.

\subsubsection{Implicit and non-additive observation}
\label{sec:implicit}
The sign-blind sensor in \cref{sec:sign-blind} is many-to-one, but it is still given
by an explicit observation map with additive Gaussian noise. We now replace this
structure by an implicit, non-additive observation relation. The observation consists
of a circle-constraint channel and a weak sign channel, given by
\begin{equation}\label{eq:obs-implicit}
\mathcal R(x,y_{\mathrm c})=x^2+y_{\mathrm c}^2-2\sim\mN(0,0.3^2I_{20});\qquad y_s=cx+\eta,\quad\eta\sim\mN(0,0.1^2I_{20}).
\end{equation}
All operations are applied coordinate-wise, and we have \(y=(y_{\mathrm c},y_{\mathrm s})\in\mathbb R^{40}\) as observation data.
In the circle channel, the
noise is imposed on the constraint \(x^2+y_{\mathrm c}^2\approx2\), rather than being
added to a forward map \(y_{\mathrm c}=H_{\mathrm c}(x)+\varepsilon\). Thus the
observation model is not naturally represented as an additive residual in the
observation variable. Moreover, the constraint is even in \(x\), so it is again
sign-blind; the weak linear channel \(y_{\mathrm s}\) is included only to make the
state identifiable over time.

This benchmark separates two levels of access to the observation model. We provide the exact observation relation in
\cref{eq:obs-implicit} to the flow-based methods, while the Kalman-type baselines
use the natural residual representation with tuned inflation and localization. In
contrast, EnCF-LF is not given the analytic form of \cref{eq:obs-implicit}. It only
receives simulator samples from the conditional observation law at forecast states and
learns a surrogate energy online. This setting is intended to mimic an observation
mechanism for which forward simulation is available but an explicit likelihood or
differentiable observation formula is not.

\begin{table}
\centering\small
\begin{tabular}{@{}lccc@{}}
\toprule
method & RMSE & CRPS & SRR\\
\midrule
EnKF (tuned) & \(0.591\pm0.127\) & \(2.23\pm0.52\) & \(0.30\pm0.03\)\\
LETKF (tuned) & \(0.750\pm0.163\) & \(2.90\pm0.71\) & \(0.27\pm0.09\)\\
% bootstrap PF (\(N=400\)) & \(0.600\pm0.075\) & \(2.29\pm0.32\) & \(0.26\pm0.04\)\\
% bootstrap PF (\(N=4000\)) & \(0.396\pm0.257\) & \(1.46\pm1.05\) & \(0.60\pm0.42\)\\
\midrule
EnSF & \(0.684\pm0.039\) & \(2.20\pm0.11\) & \(1.30\pm0.06\)\\
EnFF-OT & \(0.827\pm0.100\) & \(3.30\pm0.44\) & \(0.18\pm0.02\)\\
EnFF-F2P & \(0.717\pm0.206\) & \(2.55\pm0.74\) & \(0.53\pm0.04\)\\
\midrule
% EnCF (\(\Nreg=1\)) & \(0.480\pm0.063\) & \(1.61\pm0.18\) & \(1.66\pm0.18\)\\
EnCF & \(\mathbf{0.433\pm0.060}\) & \(\mathbf{1.47\pm0.17}\) & \(1.76\pm0.21\)\\
% EnCF-LF (\(\Nreg=1\)) & \(0.438\pm0.068\) & \(1.49\pm0.19\) & \(1.76\pm0.20\)\\
EnCF-LF & \(0.435\pm0.067\) & \(1.48\pm0.17\) & \(1.75\pm0.25\)\\
\bottomrule
\end{tabular}
\caption{Double-well dynamics with the implicit circle sensor (\(N=40\), 5 seeds). Bold marks
the best mean among the filters with the smallest seed spread (see text).}
\label{tab:implicit}
\end{table}

The results in \cref{tab:implicit} show that the controlled-flow methods are the most
effective in this implicit observation regime. With analytic access to the observation
energy, EnCF obtains substantially lower RMSE and CRPS than the classical and
existing flow-based baselines. The likelihood-free variant performs comparably,
matching the exact-energy EnCF to within a fraction of the seed spread, despite using
only simulator samples to learn the surrogate energy. This indicates that the energy-learning interface can
recover the relevant observation information without requiring an explicit additive
observation map.

\Cref{fig:implicit} explains the difference in calibration behavior. The
Kalman-type filters and the EnFF variants have small SRR values in
\cref{tab:implicit}, indicating under-dispersion in this multimodal implicit problem.
For EnKF, the ensemble visualization shows that this under-dispersion corresponds
to a collapse onto a single well for each coordinate. Once the Gaussian update
selects the wrong well, the small ensemble spread makes subsequent recovery
difficult. By contrast, EnCF and EnCF-LF are somewhat over-dispersed according to
SRR, but their lower CRPS suggests that the retained ensemble diversity is useful
rather than spurious. The ensemble keeps mass near both wells while still
concentrating around the sign supported by the observations. This behavior is
consistent with the two-valued structure induced by the circle constraint and
illustrates why controlled posterior transport is beneficial for implicit many-to-one
observations.

\begin{figure}
\centering
\includegraphics[width=\linewidth]{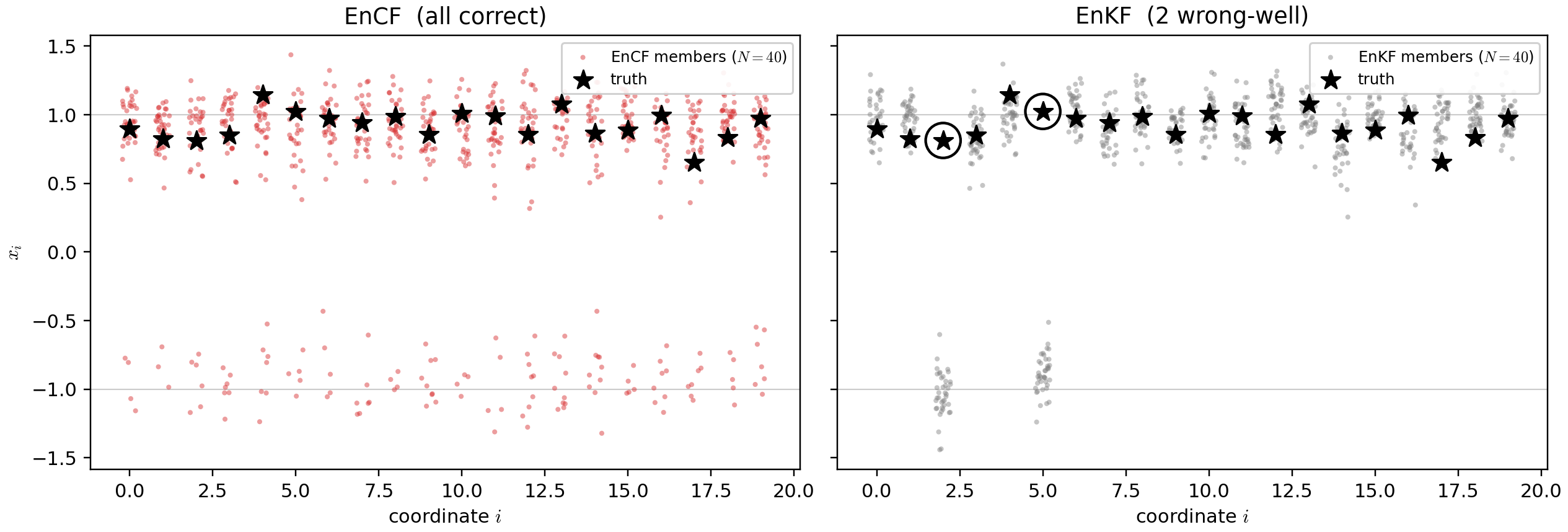}
\caption{The analysis ensemble of \(40\) members at
the final assimilation step (\(\star\): truth, \(n=20\)). \emph{Left:} EnCF keeps mass in \emph{both}
wells of the sign-blind posterior while centering on the correct sign for every coordinate.
\emph{Right:} the EnKF's Gaussian update collapses the ensemble onto a single well per
coordinate and lands in the wrong one on two of them (circled).}
\label{fig:implicit}
\end{figure}

\subsection{Computational cost}
\label{sec:cost}
The improved flexibility of EnCF comes with the additional cost of learning a
control during each analysis step. We therefore report wall-clock time separately
from the accuracy metrics above. The comparison is meant to quantify the practical
overhead of the learned controlled flow, rather than to claim that EnCF is the
cheapest filter. All timings are measured under the same implementation and
hardware environment. Since absolute runtime depends on implementation details,
we focus on the relative cost across methods.

\Cref{fig:ekcost} separates the transient and steady-state costs of the learned
controlled flow. In the early cycles, EnCF and EnCF-LF are substantially more
expensive because the control is trained with a larger epoch budget. After the cosine
schedule reaches \(E_k=1\), the per-cycle cost becomes stable. EnCF-LF remains more
expensive than EnCF in this steady regime because it also learns a surrogate
observation energy, while EnKF is the cheapest method and the existing flow-based
filters have lower per-cycle cost than the controlled-flow methods.

The right panel relates this steady-state cost to the converged RMSE on the implicit
circle benchmark. The Kalman-type filters and existing flow-based filters are faster,
but they have larger errors in this implicit many-to-one setting. EnCF and EnCF-LF
require higher wall-clock time per assimilation cycle, but they occupy the low-error
region of the plot. Thus the learned control introduces a measurable computational
overhead, while providing a favorable accuracy-cost trade-off in the observation
regime for which the method is designed.

\begin{figure}
\centering
\begin{minipage}{.48\textwidth}
\includegraphics[width=\linewidth]{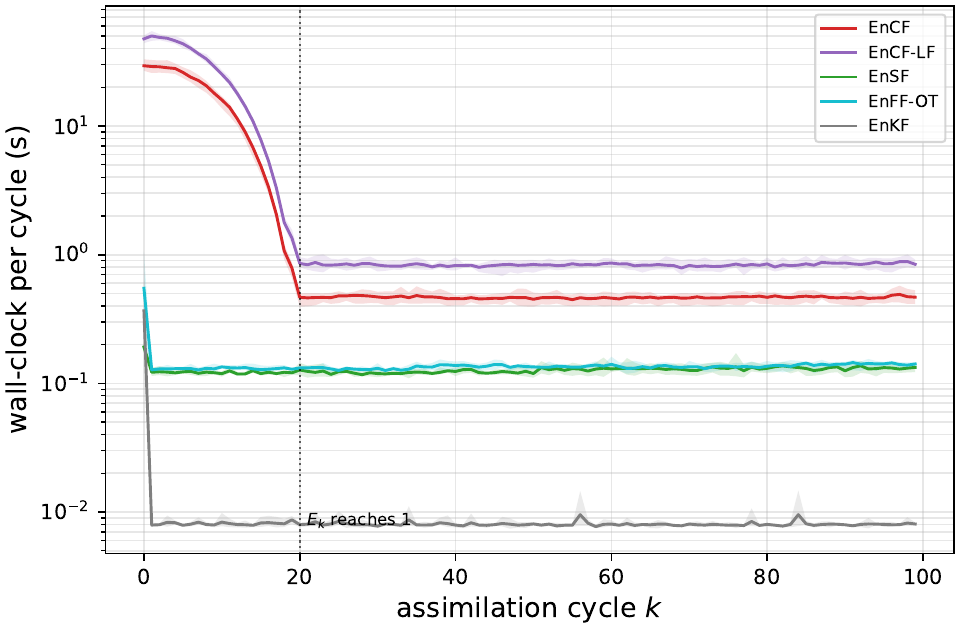}
\end{minipage}%
\hfill
\begin{minipage}{.48\textwidth}
\includegraphics[width=\linewidth]{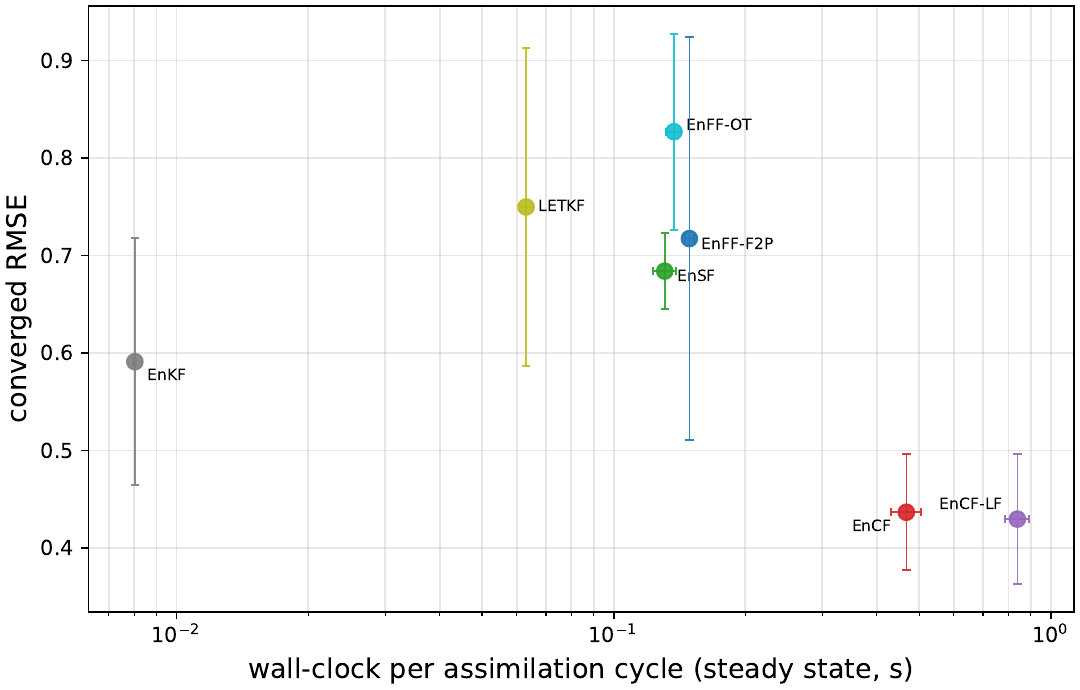}
\end{minipage}
\caption{Computational cost on the implicit circle benchmark.
\emph{Left:} wall-clock time per assimilation cycle over the rollout. EnCF and
EnCF-LF use the cosine training schedule described in \cref{sec:filter-configurations},
which decreases the number of control-training epochs from \(100\) to \(1\) during
the first \(20\) cycles.
\emph{Right:} converged RMSE versus steady-state wall-clock time per assimilation
cycle. Each marker corresponds to one filter after the EnCF training schedule has
annealed to \(E_k=1\). Shaded bands and error bars indicate variability over five
random seeds.}
\label{fig:ekcost}
\end{figure}

\section{Discussion}\label{sec:discussion}

This work introduces implicit data assimilation as a formulation for incorporating
observations that are not naturally represented by an additive residual or an explicit
forward operator. By expressing the analysis distribution as an energy tilt of the
forecast law, the observation model is separated from the numerical mechanism used
to realize the update. EnCF implements this update through a stochastic controlled
flow, while EnCF-LF replaces an unavailable analytic energy with a surrogate learned
from simulator samples. The accompanying analysis establishes ideal exactness and
organizes the practical filtering error into contributions from the base flow, the
learned control, the surrogate energy, and temporal discretization. Together, these
results provide a numerical framework for constructing ensemble analyses from
energy-defined observational information.

A broader implication is that the interface between a scientific measurement model
and an ensemble filter need not be restricted to residuals, moments, or pointwise
likelihood weights. These representations are highly effective when their underlying
assumptions are appropriate, and the results confirm that tuned Kalman-type methods
remain preferable for smooth additive-Gaussian observations. The role of EnCF is
therefore complementary. It is intended for problems in which the principal difficulty
lies in the geometry or accessibility of the observation model, including many-to-one,
multimodal, implicit, and simulator-defined relations. In these regimes, controlled
transport offers a way to retain distinct observation-compatible states rather than
compressing the analysis into a local Gaussian summary or requiring a separately
derived guidance rule for each sensing mechanism.

The flexibility of this formulation also creates several opportunities for systematic
algorithm design. Because the forecast representation, observation energy, learned
control, and transport discretization enter the analysis through distinct components,
each can be improved without changing the overall filtering principle. The error
decomposition developed in this work makes this modularity explicit and identifies
how ensemble quality, control optimization, energy approximation, and numerical
resolution affect the resulting analysis. It therefore provides a useful basis for
allocating computational effort across assimilation cycles and for adapting the method
to sharply concentrated observations, evolving forecast distributions, and spatially
coupled systems. For EnCF-LF in particular, the separation between energy learning
and controlled transport allows increasingly expressive observation models to be
incorporated while preserving the same analysis solver.

These observations point toward several directions for further development. Adaptive
schemes could select the base flow and allocate control optimization and transport
resolution according to the difficulty of each analysis step. Localized, reduced-order,
or multiscale parameterizations of the control may extend the approach to substantially
larger systems. For likelihood-free assimilation, richer conditional energy models
could accommodate non-Gaussian, spatially dependent, and temporally correlated
observation processes while retaining forward-only simulator access. More generally,
the decomposition into a forecast-generating base flow, an observation energy, and a
learned control provides a modular foundation on which alternative transport
dynamics, energy models, and approximation strategies can be developed. This
modularity, rather than the replacement of established filters in their natural regime,
is the main opportunity offered by the controlled-flow perspective.

\section*{Reproducibility}
All the codes for data generation, network training, and visualization are publicly available on the GitHub repository \texttt{\url{https://github.com/zylipku/EnCF}}.

\section*{Declarations}
% The computational work for this article was partially performed on resources of the National Supercomputing Centre, Singapore (\url{https://www.nscc.sg}).
During the preparation of this work, the authors used large language models (LLMs) to assist with the code writing and to polish the written text for spelling and grammar. The authors thoroughly reviewed and edited the content and take full responsibility for the final publication.

\section*{Acknowledgements}
Z.L. is supported by the Ministry of Education, Singapore, under its Research Centre of Excellence award to the Institute for Functional Intelligent Materials. (Project No. EDUNC-33-18-279-V12). Y.Z. is supported by the National University of Singapore and the AI for Science Institute, Beijing, through the AISI–NUS Joint Research Initiative Fund 2025 Award.

\bibliographystyle{siamplain}
\bibliography{references}
\appendix

\section{Deferred proofs}\label{app:proofs}

\subsection{Base-drift Monte-Carlo error}\label{app:mc-error}
\begin{proposition}[Monte-Carlo error of the base drift]
\label{prop:mc-error}
Fix \((z,\tau)\) with \(\beta_\tau>0\) and let
\(\hat Z_k^{(j)}\sim\pf_k\) be independent forecast samples for $j=1,\cdots,N$.
Define
\begin{equation}
\begin{aligned}
    d(z_0)\coloneqq p_{\tau\mid0}(z\mid z_0),\qquad\zeta(z_0)\coloneqq(\alpha_\tau z_0-z)/\beta_\tau^2,\quad\text{and}\\
    \hat g_N(z,\tau)\coloneqq\sum_{j=1}^N w_j\zeta(\hat Z_k^{(j)})\approx\nabla\log p_\tau(z_\tau)\eqqcolon g(z,\tau).
\end{aligned}
\end{equation}
Then \(\hat g_N\) is considered as the
self-normalized estimator of the marginal score
\(g(z,\tau)\) as in \cref{eq:EnCF-f}. Under
\cref{ass:reg}, if \(\mbE_{\pf_k}[d^2]<\infty\) and \(\mbE_{\pf_k}[\norm{\zeta}^2d^2]<\infty\), then \(\hat g_N\) is strongly
consistent and obeys the Monte-Carlo rate
\begin{equation}
\hat g_N-g=\mathcal O_p(N^{-1/2}),\qquad \mbE\norm{\hat g_N-g}^2=\mathcal O(N^{-1}).
\label{eq:mc-rate}
\end{equation}
The rate is dimension-free, but its multiplicative constant scales with the second moment of the
importance weights
\begin{equation*}
\rho\coloneqq\frac{\mbE_{\pf_k}[d^2]}{p_\tau(z)^2}=1+\chi^2\bigl(\pi^{z,\tau}\bigm\|\pf_k\bigr)\ge1,
\qquad \pi^{z,\tau}\propto \pf_kd,
\end{equation*}
the \(\chi^2\)-divergence between the posterior tilt \(\pi^{z,\tau}\) and the forecast prior, which grows as
the kernel sharpens (\(\beta_\tau\to0\), i.e.\ \(\tau\to1\)) or the state dimension increases.
\end{proposition}

\begin{proof}
The estimator \(\hat g_N\) is the self-normalized importance-sampling estimator of the ratio
\(g=\mbE_{\pf_k}[\zeta d]/\mbE_{\pf_k}[d]\). Under the stated second-moment conditions the rate
\eqref{eq:mc-rate} is the classical self-normalized importance-sampling central limit theorem, and the growth of its constant with the weight second moment \(\rho\), the
intrinsic dimension of the sampling problem, is quantified by \cite{agapiou2017importance}.
\end{proof}

\subsection{Proof of the energy-stability estimate}
\label{app:proof-energy}

\begin{proof}[Proof of \cref{prop:energy}]
The optimal control \cref{eq:ustar} is a Gibbs average of the base-kernel score
\begin{equation}
    u^\star_{J_r}(z,\tau)=\sigma(\tau)^2\frac{\int\me^{-J_r(z')}\nabla_z p_{1\mid\tau}(z'\mid z)\md z'}{\int\me^{-J_r(z')}p_{1\mid\tau}(z'\mid z)\md z'}=\sigma(\tau)^2\mbE_{\rho_r^{z,\tau}}\bigl[\nabla\log p_{1\mid\tau}(\cdot\mid z)\bigr]
\end{equation}
over \(\rho_r^{z,\tau}\). 
Differentiation under the integral sign
along the homotopy is justified by dominated convergence. The \(r\)-derivatives of the integrands are
dominated by \(\me^{-\inf_{r}J_r}\abs{J_1-J_2}(1+\norm{\nabla\log p_{1\mid\tau}})p_{1\mid\tau}\),
integrable under \cref{ass:reg} and the bounded score variance \(V\). Differentiating along \(J_r=(1-r)J_1+rJ_2\) gives
\begin{equation}
    \frac{\md u^\star_{J_r}}{\md r}=\sigma(\tau)^2\Cov_{\rho_r^{z,\tau}}\bigl(J_1-J_2,\nabla\log p_{1\mid\tau}(\cdot\mid z)\bigr)
\end{equation}
since the \(r\)-derivative of a Gibbs average is the covariance of the score with \(\partial_r J_r=J_1-J_2\).
Consequently, we have
\begin{equation}
    u^\star_{J_1}-u^\star_{J_2}=\sigma(\tau)^2\int_0^1\Cov_{\rho_r^{z,\tau}}(J_1-J_2,-\nabla\log p_{1\mid\tau})\md r.
\end{equation}
Bounding the \(r\)-integral by its supremum and applying Cauchy--Schwarz results in
\begin{equation}
\begin{aligned}
    \norm{u^\star_{J_1}-u^\star_{J_2}}_\infty&\le\sigma(\tau)^2\sup_{r\in[0,1]}\sqrt{\Var_{\rho_{J_r}}(J_1-J_2)\Tr\Var_{\rho_{J_r}}\nabla\log p_{1\mid\tau}}\\
    &\le\sigma(\tau)^2\sqrt{C_PV}\norm{\nabla J_1-\nabla J_2}_\infty.
\end{aligned}
\end{equation}
The last inequality holds by the bounded score variance \(V\) and the Poincar\'e inequality with constant \(C_P\), both uniform in \(r\) as assumed in \cref{prop:energy}.
\end{proof}

\subsection{Proof of the multi-step filtering theorem}
\label{app:proof-multistep}

\begin{proof}[Proof of \cref{thm:multistep}]
With the convention
\begin{equation}
    \pa_k=\Phi_{0\to k}(\pa_0),\qquad\hpa_k=\hat\Phi_{0\to k}(\hpa_0),\qquad\forall k=0,1,\cdots,
\end{equation}
we have
\begin{equation}
    \begin{aligned}
        \hpa_l-\pa_l&=\Phi_{l\to l}(\hpa_l)-\Phi_{0\to l}(\pa_0)\\
        &=\sum_{k=0}^{l-1}\left[\Phi_{l-k\to l}(\hpa_{l-k})-\Phi_{l-k-1\to l}(\hpa_{l-k-1})\right]+\Phi_{0\to l}(\hpa_0)-\Phi_{0\to l}(\pa_0),
    \end{aligned}
\end{equation}
and it follows by the triangle inequality that
\begin{equation}
    \begin{aligned}
        d_l&\le\sum_{k=0}^{l-1}C_{\mathrm{stab}}
    \varrho^k\TV{\hpa_{l-k}}{\Phi_{l-k}(\hpa_{l-k-1})}+C_{\mathrm{stab}}
    \varrho^l\TV{\hpa_0}{\pa_0}
    \end{aligned}
\end{equation}
by applying \cref{ass:stab}. The inequality \cref{eq:dl-upper-bound} follows immediately.
\end{proof}

\section{Experimental details}
\label{app:experiment}

\paragraph{EnCF configuration}
For the base flow in \cref{eq:EnCF-f}, we use
\begin{equation}
    \alpha_\tau=\epsilon_\alpha+(1-\epsilon_\alpha)\tau,
    \qquad
    \beta_\tau=\sqrt{1-(1-\epsilon_\beta)\tau},
    \qquad
    (\epsilon_\alpha,\epsilon_\beta)=(0.5,0.025),
\end{equation}
following \cite{bao2024ensf}. This regularization keeps the noise schedule in
\cref{eq:memoryless} bounded near \(\tau=1\). All flow-based methods use
\(N_\tau=100\) transport steps.

The number of control-training epochs at cycle \(k\) is
\begin{equation}
    E_k=
    \max\left\{
    1,\,
    \operatorname{round}\left[
    1+\frac{99}{2}
    \left(
    1+\cos\frac{\pi\min(k,20)}{20}
    \right)
    \right]
    \right\}.
\end{equation}
Thus, \(E_k\) decreases from \(100\) to \(1\) over the first \(20\) cycles and
remains \(1\) thereafter. The \(\Nreg\) regression updates associated with one
rollout are evaluated in a single batched forward pass.

\paragraph{Circle-sensor energy}
The exact conditional energy includes the state-dependent normalizer
\begin{equation}
    Z(x)=
    \int
    \exp\left[
    -\frac{(x^2+y_{\mathrm c}^2-R^2)^2}
    {2\sigma_{\mathrm o}^2}
    \right]
    \md y_{\mathrm c}.
\end{equation}
We evaluate \(Z(x)\) independently for each coordinate using a 512-point grid
over
\begin{equation}
    |y_{\mathrm c}|
    \le
    R+5\max(\sigma_{\mathrm o},10^{-3}),
\end{equation}
with \texttt{logsumexp} stabilization. Gradients of \(\log Z(x)\) are computed
by automatic differentiation through the quadrature.

\paragraph{Baseline tuning}
For the perturbed-observation EnKF and LETKF, the inflation factor is selected
from \(\{1.0,1.05,1.1,1.2\}\), while the Gaspari--Cohn localization radius is
fixed at \(4\). Each configuration is evaluated over five random seeds.

\paragraph{Numerical integration and initialization}
The Kuramoto--Sivashinsky equation is integrated spectrally with
\(\Delta t=0.25\). Observations are assimilated every four model steps after a
2000-step spin-up, and the initial ensemble is obtained by adding independent
unit Gaussian perturbations to the reference state. The mode-separation sweep
in \cref{sec:nongaussian} uses \(100\) assimilation cycles.

The double-well dynamics are integrated by Euler--Maruyama with
\(\Delta t=0.05\) and ten drift substeps per model step. The ensemble is
initialized across both wells using Gaussian perturbations with standard
deviation \(1\), following a 1000-step spin-up.

\end{document}